\documentclass{amsart}
\usepackage{booktabs} 
\usepackage{tikz-cd}
\usepackage{thmtools}
\usepackage{xspace}
\usepackage{enumerate}
\usepackage{amsmath,amssymb}

\newtheorem{theorem}{Theorem}[section]
\newtheorem{lemma}[theorem]{Lemma}
\newtheorem{proposition}[theorem]{Proposition}
\newtheorem{corollary}[theorem]{Corollary}
\newtheorem{notation}[theorem]{Notation}
\theoremstyle{definition}
\newtheorem{definition}[theorem]{Definition}
\newtheorem{example}[theorem]{Example}

\theoremstyle{remark}
\newtheorem{remark}[theorem]{Remark}

\numberwithin{equation}{section}

\begin{document}

\newcommand {\AT} {\texttt{\textup{Atoms}}}
\newcommand {\Args} {\texttt{\textup{Args}}}
\newcommand {\Arg} {\texttt{\textup{Arg}}}
\newcommand {\Att} {\texttt{\textup{Att}}}
\newcommand {\Def} {\texttt{\textup{Def}}}
\newcommand {\Con} {\texttt{\textup{Con}}}
\newcommand {\Cons} {\texttt{\textup{Cons}}}
\newcommand {\Sup} {\texttt{\textup{Sup}}}
\newcommand {\Exts} {\texttt{\textup{Exts}}}
\newcommand {\Acc} {\texttt{\textup{Acc}}}
\newcommand {\NotAcc} {\texttt{\textup{NotAcc}}}
\newcommand {\sem} {\texttt{\textup{sem}}}
\newcommand {\Subs} {\texttt{\textup{Subs}}}
\newcommand {\DefBy} {\texttt{\textup{DefBy}}}
\newcommand {\NotDef} {\texttt{\textup{NotDef}}}
\newcommand {\Rep} {\texttt{\textup{Repairs}}}
\newcommand {\post} {\texttt{\textup{Post}}}

\newcommand {\ap} {\texttt{\textup{AP}}}
\newcommand {\fp} {\texttt{\textup{FP}}}
\newcommand {\ad} {\texttt{\textup{Ad}}}
\newcommand {\te} {\texttt{\textup{te}}}
\newcommand {\prof} {\texttt{\textup{P}}}
\newcommand {\lect} {\texttt{\textup{Le}}}
\newcommand {\base} {\texttt{\textup{Base}}}
\newcommand {\subs} {\texttt{\textup{Subs}}}
\newcommand {\cl} {\texttt{\textup{Cl}}}
\newcommand {\set} {\texttt{\textup{Set}}}
\newcommand {\conflict} {\texttt{\textup{Conflicts}}}
\newcommand {\mods} {\texttt{\textup{mods}}}
\newcommand {\ans} {\texttt{\textup{ans}}}
\newcommand {\chase} {\texttt{\textup{chase}}}
\newcommand {\adm} {\texttt{\textup{adm}}}
\newcommand {\grd} {\texttt{\textup{grd}}}
\newcommand {\stb} {\texttt{\textup{stb}}}
\newcommand {\prf} {\texttt{\textup{prf}}}
\newcommand {\cmp} {\texttt{\textup{cmp}}}
\newcommand {\sce} {\texttt{\textup{sc}}}
\newcommand {\cre} {\texttt{\textup{cr}}}
\newcommand {\ar} {\texttt{\textup{AR}}}
\newcommand {\iar} {\texttt{\textup{IAR}}}
\newcommand {\bra} {\texttt{\textup{BR}}}

\newcommand {\emp} {\texttt{\textup{Em}}}
\newcommand {\staff} {\texttt{\textup{StaffCS}}}
\newcommand {\rese} {\texttt{\textup{Re}}}
\newcommand {\ta} {\texttt{\textup{TA}}}
\newcommand {\tenu} {\texttt{\textup{tenured}}}
\newcommand {\teAs} {\texttt{\textup{taOf}}}
\newcommand {\uc} {\texttt{\textup{UC}}}
\newcommand {\gc} {\texttt{\textup{GC}}}
\newcommand {\vi} {\texttt{\textup{v}}}
\newcommand {\kr} {\texttt{\textup{KR}}}
\newcommand {\kd} {\texttt{\textup{KD}}}
\newcommand {\ml} {\texttt{\textup{ML}}}

\newcommand {\TG} {\texttt{\textup{TG}}}
\newcommand {\CO} {\texttt{\textup{C}}}
\newcommand {\ID} {\texttt{\textup{ID}}}
 \newcommand {\x} {\Vec{x}}
 \newcommand {\y} {\Vec{y}}
\newcommand {\D} {\texttt{\textup{D}}}
\newcommand {\C} {\texttt{\textup{C}}}
\newcommand {\tL} {\texttt{\textup{L}}}
\newcommand {\tT} {\texttt{\textup{T}}}
\newcommand {\tS} {\texttt{\textup{S}}}

\newcommand {\body} {\texttt{\textup{body}}}
\newcommand {\head} {\texttt{\textup{head}}}
\newcommand {\fa} {\texttt{\textup{fact}}}
\newcommand {\cla} {\texttt{\textup{claim}}}
\newcommand {\off} {\texttt{\textup{offer}}}
\newcommand {\why} {\texttt{\textup{why}}}
\newcommand {\whyn} {\texttt{\textup{whynot}}}
\newcommand {\cont} {\texttt{\textup{contrary}}}
\newcommand {\cond} {\texttt{\textup{concede}}}
\newcommand {\ret} {\texttt{\textup{retract}}}
\newcommand {\pas} {\texttt{\textup{pass}}}

\newcommand{\cn}{\texttt{\textup{CN}}}
\newcommand{\cnb}{\overline{\cn}}
\newcommand{\pr}{\texttt{\textup{Pr}}}
\newcommand{\lf}{\mathcal{LF}}
\newcommand{\mX}{\mathcal{X}}
\newcommand{\mL}{\mathcal{L}}
\newcommand{\mF}{\mathcal{F}}
\newcommand{\mK}{\mathcal{K}}
\newcommand{\mC}{\mathcal{C}}
\newcommand{\mR}{\mathcal{R}}
\newcommand{\mI}{\mathcal{I}}
\newcommand{\mP}{\mathcal{P}}
\newcommand{\mD}{\mathcal{D}}
\newcommand{\mAF}{\mathcal{AF}}
\newcommand{\datalogPM}{$\text{Datalog}^\pm$}
\newcommand{\mB}{\mathcal{B}}
\newcommand{\mS}{\mathcal{S}}
\newcommand{\mA}{\mathcal{A}}
\newcommand{\mU}{\mathcal{U}}
\newcommand{\mT}{\mathcal{T}}
\newcommand{\mE}{\mathcal{E}}
\newcommand{\mG}{\mathcal{G}}

\newcommand {\rl} {\texttt{\textup{Rule}}}
\newcommand {\child} {\texttt{\textup{Child}}}

\newcommand{\mDE}{\mathcal{DE}}
\newcommand{\mCU}{\mathcal{CU}}
\newcommand{\FR}{\texttt{\textup{FRe}}}

\newcommand {\msc} {\texttt{\textup{MCS}}}
\newcommand {\para} {\texttt{\textup{para}}}
\newcommand {\f} {\texttt{\textup{fa}}}
\newcommand {\nf} {\texttt{\textup{nf}}}
\newcommand {\um} {\texttt{\textup{um}}}
\newcommand {\op} {\texttt{\textup{O}}}
\newcommand {\po} {\texttt{\textup{P}}}
\newcommand {\pl} {\texttt{\textup{pl}}}
\newcommand {\PL} {\texttt{\textup{PL}}}
\newcommand {\argu} {\texttt{\textup{arg}}}
\newcommand {\G} {\texttt{\textup{G}}}
\newcommand {\tg} {\texttt{\textup{ta}}}
\newcommand {\id} {\texttt{\textup{id}}}

\newcommand{\NT}{\ensuremath{\mathsf{N_t}}\xspace}
\newcommand{\skol}{\textsf{sk}}

\title{Dialogue-based Explanations for Logical Reasoning using Structured Argumentation}

\author{Loan Ho}
\email{loanthuyho.cs@gmail.com}
\author{Stefan Schlobach}
\email{k.s.schlobach@vu.nl}
 \address{Vrije University Amsterdam, The Netherlands}

%

%

\begin{abstract}
  
  The problem of explaining inconsistency-tolerant reasoning in knowledge bases (KBs) is a prominent topic in Artificial Intelligence (AI). While there is some work on this problem, the explanations provided by existing approaches often lack critical information or fail to be expressive enough for non-binary conflicts. In this paper, we identify structural weaknesses of the state-of-the-art and propose a generic argumentation-based approach to address these problems. This approach is defined for logics involving reasoning with maximal consistent subsets and shows how any such logic can be translated to argumentation. Our work provides dialogue models as dialectic-proof procedures to compute and explain a query answer wrt inconsistency-tolerant semantics. This allows us to construct dialectical proof trees as explanations, which are more expressive and arguably more intuitive than existing explanation formalisms.
  
  
  %

\end{abstract}
\maketitle   

\section{Introduction}
This paper addresses the problem of explaining logical reasoning in (inconsistent) KBs. Several approaches have been proposed by~\cite{Lukasiewicz2020,Thomas2022Neg,Meghyn2019,Alrabbaa2022}, which mostly include \emph{set-based explanations} and \emph{proof-based explanations}.
Set-based explanations, which are responsible for the derived answer, are defined as \emph{minimal sets of facts} in the existential rules~\cite{Lukasiewicz2020,Thomas2022Neg} or as \emph{causes} in Description Logics (DLs)~\cite{Meghyn2019}.
 Additionally, the work in \cite{Meghyn2019} provides the notion of \textit{conflicts} that are \emph{minimal sets of assertions} responsible for a KB to be inconsistent. 
Set-based explanations present the necessary premises of entailment and, as such, do not articulate the (often non-obvious) reasoning that connects those premises with the conclusion nor track conflicts. 
Proof-based explanations provide graphical representations to allow users to understand the reasoning progress better~\cite{Alrabbaa2022}. Unfortunately, the research in this area generally focuses on reasoning in \emph{consistent} KBs.

The limit of the above approaches is that they lack the tracking of contradictions, whereas argumentation can address this issue.
Clearly, argumentation offers a potential solution to address inconsistencies. Those are divided into three approaches:
\begin{itemize}
       \item \emph{Argumentation approach based on Deductive logic}: Various works propose instantiations of abstract argumentation (AFs) for  \datalogPM~\cite{ARIOUA201776,Yun2020SetsOA}, Description Logic~\cite{ZhangL13} or Classical Logic~\cite{DAgostinoM18}, focusing on the translation of KBs to argumentation without considering explanations.
       In~\cite{LoanHo2022}, explanations can be viewed as dialectical trees defined abstractly, requiring a deep understanding of formal arguments and trees, making the work impractical for non-experts.
       In~\cite{Yun2020SetsOA,yun2018}, argumentation with collective attacks is proposed to capture non-binary conflicts in \datalogPM, i.e., assuming that every conflict has more than two formulas.

        \item \emph{Sequent-based argumentation}~\cite{ArieliS19} and its extension (\emph{Hypersequent-based argumentation})~\cite{BorgAS17}, using \textit{Propositional Logic}, provide non-monotonic extensions for Gentzen-style proof systems in terms of argumen-tation-based.
        Moreover, the authors conclude by wishing future work to include “the study of more expressive formalisms, like those that are based on first-order logics”~\cite{ArieliS19}.
       
       \item \emph{Rule-based argumentation}: 
       \emph{DeLP/DeLP} with \emph{collective attacks} are introduced for defeasible logic programming~\cite{Alejandro2014,AlsinetBG10}. However, in~\cite{Yun2020SetsOA}, the authors claim that they cannot instantiate DeLP for \datalogPM, since DeLP only considers ground rules.
       In \cite{Prakken2002,ModgilP14}, \emph{ASPIC/ASPIC+} is introduced for defeasible logic. Following~\cite{Amgoud12}, the logical formalism in ASPIC+ is ill-defined, i.e., the contrariness relation is not general enough to consider n-ary constraints. This issue is stated in ~\cite{Yun2020SetsOA} for \datalogPM, namely, the ASPIC+ cannot be directly instantiated with Datalog. The reasons behind this are that Datalog does not have the negation and the contrariness function of ASPIC+ is not general for this language.

        Notable works include assumption-based argumentation (ABA)~\cite{Dung2009} and ABA with collective attack~\cite{DimopoulosD0R0W24}, which are applied for default logic and logic programming. However, ABAs ignore cases of the inferred assumptions conflicting, which is allowed in the existential rules, Description Logic and Logic Programming with Negation as Failure in the Head. We call the ABAs "\emph{flat} ABAs".       
        In~\cite{SCHULZ_TONI_2016}, "flat" ABAs link to Answer Set Programming but only consider a single conflict for each assumption. In~\cite{Rapberger2024,Lehtonen2024},
        "\emph{Non-flat}" ABAs overcome the limits of "flat" ABAs, which allow the inferred assumptions to conflict. However, like ASPIC+, the non-flat ABAs ignore the n-ary constraints case.
        \emph{Contrapositive} ABAs~\cite{HEYNINCK2020103} and its collective attack version~\cite{ArieliH24}, which use \emph{contrapositive propositional logic}, propose extended forms for  'flat' and 'non-flat' ABAs. While~\cite{ArieliH24} mainly focuses on representation (which can be simulated in our setting, see Section~\ref{subsec:relation-framework}), we extend our study to proof procedures in AFs with collective attacks.

\end{itemize}

Argumentation offers dialogue games to determine and explain the acceptance of propositions for classical logic~\cite{Castagna21}, for \datalogPM ~\cite{ARIOUA2017244,Arioua2014FE,Arioua2016,ARIOUA201776}, for logic programming/ default logic~\cite{ThangDH12,Xiuyi14}, and for defeasible logic~\cite{Prakken05}. However, the models have limitations.
In~\cite{ARIOUA2017244,Arioua2014FE}, the dialogue models take place between a domain expert but are only applied to a specific domain (agronomy). 
In~\cite{Arioua2016,ARIOUA201776,Prakken05,Castagna21}, persuasion dialogues (dialectic proof procedures) generate the abstract dispute trees defined abstractly that include arguments and attacks and ignore the internal structure of the argument. These works lack exhaustive explanations, making them insufficient for understanding inference steps and argument structures. The works in~\cite{DUNG2006114,DUNG2007642} provide dialectical proof procedures, while the works~\cite{Xiuyi14,ThangDH12} offer dialogue games (as a distributed mechanism) for "flat" ABAs to determine sentence acceptance under (grounded/ admissible/ ideal) semantics. Although the works in~~\cite{Xiuyi14,ThangDH12} are similar to our idea of using dialogue and tree, these approaches do not generalize to n-ary conflicts.

The existing studies are mostly restricted to (1) specific logic or have limitations in representation aspects, (2) AFs with binary conflicts, and (3) lack exhaustive explanations. 
This paper addresses the limitations by introducing a general framework that provides dialogue models as dialectical proof procedures for acceptance in structured argumentation.
The following is a simple illustration of how our approach works in a university example.

\begin{example}
\label{ex:motivation-ex}
Consider inconsistent knowledge about a university domain, in which we know that: \emph{lecturers} $(\lect)$ and \emph{researchers} $(\rese)$ are \emph{employers} $(\emp)$; \emph{full professors} $(\fp)$ are researchers; everyone who is a teaching assistant $(\teAs)$ of an \emph{undergraduate course} $(\uc)$ is a teaching assistant $(\ta)$; everyone who teaches a course is a lecturer and everyone who teaches a \emph{graduate course} $(\gc)$ is a full professor. 
However, teaching assistants can be neither researchers nor lecturers, which leads to inconsistency. 
We also know that an individual \emph{Victor}  apparently is or was a teaching assistant of the KD course $(\teAs(\vi, \kd))$, and the KD course is an undergraduate course $(\uc(\kd))$. Additionally, \emph{Victor} teaches either the KD course $(\te(\vi, \kd))$ or the KR course $(\te(\vi, \kr))$, where the KR course is a graduate course $(\gc(\kr))$.
The KB $\mK_1$ is modelled as follows:
\begin{align*}
\mF_1 = & \{\teAs(\vi,\kd),\ \te(\vi, \kd),\ \uc(\kd),\ \te(\vi,\kr),\ \gc(\kr) \} \\
\mR_1 = & \{ r_1:\ \lect(x) \rightarrow \emp(x) ,\
  r_2:\ \rese(x) \rightarrow \emp(x),\ \\
& r_3:\ \fp(x) \rightarrow \rese(x) ,\
  r_4:\ \teAs(x,y) \land \uc(y)  \rightarrow \ta(x),\ \\
& r_5:\ \te(x,y) \rightarrow \lect(x),\
  r_6:\ \te(x,y) \land \gc(y)  \rightarrow \fp(x) \} \\
\mC_1 = & \{ c_1:\ \ta(x) \land \rese(x) \rightarrow \bot,\
  c_2:\ \lect(x) \land \ta(x) \rightarrow \bot\}
\end{align*}
When a user asks "\emph{Is Victor a researcher?}", the answer will be "\emph{Yes, but Victor is possibly a researcher}". 
The current method~\cite{Lukasiewicz2020,Thomas2022Neg,Meghyn2019} will provide a set-based explanation consisting of (1) the \emph{cause} $\{\te(\vi,\kr),\ \gc(\kr)\}$ entailing the answer (why the answer is accepted) and (2) the \emph{conflict} $\{\teAs(\vi,\kd), \uc(\kd) \}$ being inconsistent with every cause (why the answer cannot be \textit{accepted}).
The cause, though, does not show a series of reasoning steps to reach $\rese(\vi)$ from the justification $\{\te(\vi,\kr),\ \gc(\kr)\}$.
The conflict still lacks \emph{all relevant information} to explain this result.
Indeed, in the KB, the fact $\te(\vi, \kd)$ deducing $\lect(\vi)$ makes the conflict $\{\teAs(\vi,\kd), \uc(\kd) \}$ deducing $\ta(\vi)$ \emph{uncertain}, due to the constraint that lecturers cannot be teaching assistants. Thus, using the conflict in the explanation is \emph{insufficient} to assert the non-acceptance of the answer. 
It remains unclear why the answer is possible.
Instead, $\te(\vi, \kd)$ deducing $\lect(\vi)$ should be included in the explanation.
Without knowing the relevant information, it is impossible for the user - especially non-experts in logic - to understand why this is the case.

However, using the argumentation approach will provide a dialogical explanation that is more informative and intuitive.
The idea involves a dialogue between a proponent and opponent, where they exchange logical formulas to a dispute agree.
The proponent aims to prove that the argument in question is acceptable, while the opponent exhaustively challenges the proponent’s moves. The dialogue where the proponent wins represents a proof that the argument in question is accepted. The dialogue whose graphical representation is shown in Figure~\ref{fig:tree-user} proceeds as follows:

1. Suppose that the proponent wants to defend their claim $\rese(\vi)$. They can do so by putting forward an argument, say $A_1$, supported by facts $\te(\vi, \kr)$ and $\gc(\kr)$: 
\begin{align*}
     A_1 :\ & \rese(\vi) (\text{by } r_3) \\
          & \fp(\vi) (\text{by } r_6) \\
          & \te(\vi, \kr), \gc(\kr) (\text{by facts})\\
    \emph{Proponent: }  & \emph{I believe that } \vi \emph{ is a researcher because he is a full professor; and given the fact }\\
    & \emph{that he teaches the KR course and KR is a graduate course.}
\end{align*}

2. The opponent challenges the proponent’s argument by attacking the claim $\rese(\vi)$ with an argument $\ta(\vi)$, say $A_2$, supported by facts $\teAs(\vi, \kd)$ and $\uc(\kd)$:
\begin{align*}
    A_2:\ & \ta(\vi) (\text{by } r_4) \\
         & \teAs(\vi, \kd), \uc(\kd) (\text{by facts})\\
  \emph{Opponent: }    &  \vi \emph{ is not possibly a researcher because } \vi \emph{ is a TA given the fact that he is a TA }\\
    & \emph{of the KD course and  KD is an undergraduate course.}
\end{align*}

3. To argue that the opponent's attack is not possible - and to further defend the initial claim $\rese(\vi)$ - the proponent can counter the opponent's argument by providing additional evidence $\lect(\vi)$ supported by a fact $\te(\vi, \kd)$:
%
%
\begin{align*}
     A_3 :\ & \lect(\vi) (\text{by } r_5) \\
           & \te(\vi, \kd)(\text{by facts})\\
     \emph{Proponent: } & \emph{$\vi$ is not possibly a TA because } \vi \emph{ is certainly a lecturer given the fact that he } \\
    & \emph{teaches the KD course.}
\end{align*}

4. The opponent concedes $\rese(\vi)$ since it has no argument to argue the proponent.
\begin{align*}
     \emph{Opponent: } & \emph{I concede  that $\vi$ is a researcher because I have no argument to argue that } \vi \emph{ is } \\
    & \emph{neither a lecturer nor researcher.}
\end{align*}
The proponent's belief $\rese(\vi)$ is defended successfully, namely, $\rese(\vi)$ that is justified by facts $\{\te(\vi, \kr)$, $\gc(\kr)\}$ that be extended to be the defending set $\{\te(\vi, \kr)$, $\gc(\kr)$, $\te(\vi, \kd)\}$ that can counter-attack every attack.

By the same line of reasoning, the opponent can similarly defend their belief in the contrary statement $\ta(\vi)$ based on the defending set~$\{\teAs(\vi, \kd), \uc(\kd)\}$.
Because different agents can hold contrary claims, the acceptance semantics of the answer can be considered \emph{credulous} rather than \emph{sceptical}. In other words, the answer is deemed \emph{possible} rather than \emph{plausible}. Thus, the derived system can conclude that $\vi$ \emph{is possibly a researcher}.
    
\end{example}



The main contributions of this paper are the following:
\begin{itemize}
    \item We propose a \textbf{proof-oriented (logical) argumentation framework with collective attacks} (P-SAF), 
    in which we consider abstract logic to generalize monotonic and non-monotonic logics involving reasoning with maximal consistent subsets, and we show how any such logic can be translated to P-SAFs.
    We also conduct a detailed investigation of how existing argumentation frameworks in the literature can be instantiated as P-SAFs.
    Thus, we demonstrate that the P-SAF framework is sufficiently generic to encode n-ary conflicts and to enable logical reasoning with (inconsistent) KBs.

    \item We introduce a \textbf{novel explanatory dialogue model} viewed as a dialectical proof procedure to compute and explain the \textit{credulous}, \textit{grounded} and \textit{sceptical} acceptances in P-SAFs. The dialogues, in this sense, can be regarded as explanations for the acceptances. As our main theoretical result, we prove the soundness and completeness of the dialogue model wrt argumentation semantics.
    
    
    \item This novel explanatory dialogue model provides \textbf{dialogical explanations} for the acceptance of a given query wrt inconsistency-tolerant semantics, and  \textbf{dialogue trees} as graphical representations of the dialogical explanations. Based on these dialogical explanations, our framework assists in understanding the intermediate steps of a reasoning process and enhancing human communication on logical reasoning with inconsistencies.
    \end{itemize}

\section{Preliminaries}
\label{sec:preliminary}

To motivate our work, we review argumentation approaches using Tarski abstract logic characterized by a consequence operator~\cite{Amgoud2009}. 
However, many logic in argumentation systems, like ABA or ASPIC systems, do not always impose certain axioms, such as the absurdity axiom.
Defining the consequence operator by means of "models" cannot allow users to understand reasoning progresses better, as inference rule steps are implicit. 
These motivate a slight generalization of consequence operators in a proof-theoretic manner, inspired by~\cite{Stephen1975}, with minimal properties.

Most of our discussion applies to abstract logics (monotonic and non-monotonic) which slightly generalize Tarski abstract logic.
Let $\mL$ be a set of \emph{well-formed formulas}, or simply \textit{formulas}, and $X$ be an arbitrary set of formulas in $\mL$. With the help of \emph{inference rules}, new formulas are derived from $X$; these formulas are called \emph{logical consequences} of $X$; a \emph{consequence operator} (called \textit{closure operator}) returns the logical consequences of a set of formulas. 
    
\begin{definition}
     We define a map $\cn : 2^{\mL} \to 2^{\mL}$ such that $\cnb(X) = \bigcup_{n \geq 0}\cn^{n}(X)$ satisfies the axioms: 
     \begin{itemize}
        \item ($A_1$) \textbf{Expansion} $X \subseteq \cnb(X)$.
        \item ($A_2$) \textbf{Idempotence} $\cnb(\cnb(X)) = \cnb(X)$.

    \end{itemize}
\end{definition}
In general, a map $2^{\mL} \to 2^{\mL}$ satisfying these axioms $A_1 -\ A_2$ is called a \emph{consequence operator}. Other properties that consequence operators might have, but that we do not require in this paper, are
\begin{itemize}   
    \item ($A_3$) \textbf{Finiteness} $\cnb(X) \subseteq \bigcup \{ \cnb(Y) \mid Y \subset_{f} X \}$ where the notation $Y \subset_{f} X$ means that $Y$ is a finite proper subset of $X$.
    \item ($A_4$) \textbf{Coherence} $\cnb(\emptyset) \neq \mL$.
    \item ($A_5$) \textbf{Absurdity} $\cnb(\{x\}) = \mL$ for some $x$ in the language $\mL$.
\end{itemize}

Note that finiteness is essential for practical reasoning and is satisfied by any logic that has a decent proof system.

An \emph{abstract logic} includes a pair $\mL$ and a consequence operator $\cn$. Different logics have consequence operators with various properties that can satisfy certain axioms. For instance, the class of Tarskian logics, such as classical logic, is defined by a consequence operator satisfying $A_1 -\ A_5$ while the one of defeasible logic satisfies $A_1 -\ A_3$.

\begin{example} An inference rule $r$ in first-order logic is of the form $\frac{ p_1, \ldots, p_n}{c}$ where its conclusion is $c$ and the premises are $p_1, \ldots, p_n$. $c$ is called a \emph{direct consequence} of $p_1, \ldots, p_n$ by virtue of $r$. If we define $\cn(X)$ as the set of direct consequences of $X \subseteq \mL$, then $\cnb$ coincides with
$\cnb(X) = \{ \alpha \in \mL \mid X \vdash \alpha \}$ and
satisfies the axioms $A_1 -\ A_5$.
\end{example}

 Fix a logic $(\mL, \cn)$ and a set of formulas $X \subseteq \mL$. We say that:
\begin{itemize} 
    \item $X$ is \textit{consistent} wrt $(\mL, \cn)$ iff $\cnb(X) \neq \mL$. It is \textit{inconsistent} otherwise;   
     \item $X$ is a \textit{minimal conflict} of $\mK$ if $X^{\prime}  \subsetneq X$ implies  $X^{\prime}$ is consistent.
    \item A \textit{knowledge base} (KB) is any subset $\mK$ of $\mL$. Formulas in a KB are called \emph{facts}. A knowledge base may be inconsistent. 
   
\end{itemize}

Reasoning in inconsistent KBs $\mK \subseteq \mL$ amounts to:
\begin{enumerate}
    \item Constructing \emph{maximal consistent subsets},
    \item Applying \emph{classical entailment mechanism} on a choice of the maximal consistent subsets. 
\end{enumerate}
Motivated by this idea, we give the following definition.

\begin{definition} Let $\mK$ be a KB and $X \subseteq \mK$ be a set of formulas. $X$ is a \emph{maximal (for set-inclusion) consistent subsets} of $\mK$ iff 
\begin{itemize}
    \item $X$ is consistent,
    \item there is no $X^{\prime}$ such that $X \subset X^{\prime}$ and $X^{\prime}$ is consistent.
\end{itemize}
 We denote the set of all maximal consistent subsets by $\msc(\mK)$.
\end{definition}

Inconsistency-tolerant semantics allow us to determine different types of entailments.
\begin{definition}
Let $\mK$ be a KB. A formula $\phi \in \mL$ is entailed in
\begin{itemize}
    \item  \emph{some maximal consistent subset}  iff for some $\Delta \in \msc(\mK)$, $\phi \in \cnb(\Delta)$;

    \item the \emph{intersection of all maximal consistent subsets} iff for $\Psi = \bigcap \{\Delta \mid \Delta \in \msc(\mK)\}$, $\phi \in \cnb(\Psi)$;

    \item   \emph{all maximal consistent subsets} iff for all $\Delta \in \msc(\mK)$, $\phi \in \cnb(\Delta)$.
\end{itemize}  
\end{definition}

Informally, \textit{some maximal consistent subset semantics} refers to \textbf{possible answers}, \textit{all maximal consistent subsets semantics} to \textbf{plausible answers}, and the \emph{intersection of all maximal consistent subsets semantics} to \textbf{surest answers}.

In the following subsections, we illustrate the generality of the above definition by providing instantiations for propositional logic, defeasible logic, \datalogPM.
Table~\ref{tab:properties} summarizes properties holding for consequence operators of the instantiations.

\begin{table} 
\centering
 \begin{tabular}{c l l l l l}
   \toprule
   Axioms & $\cn_c$  & $\cn_{p}$ & $\cn_{d}$ & $\cn_{co}$ & $\cn_{da}$ \\
   \midrule
   $A_1$ & $\times$ & $\times$ & $\times$ & $\times$   &  $\times$  \\
   $A_2$ & $\times$ & $\times$ & $\times$  & $\times$  & $\times$ \\
   $A_3$ & $ \times $ & $\times$ & $\times$ & $\times$ &   \\
   $A_4$ & $\times$ & $\times$ &  &  &    \\
   $A_5$ & $\times$ & $\times$ &  &  &  \\
   \bottomrule
 \end{tabular}
 \caption{Properties of consequence operators of the instantiations}
 \label{tab:properties} 
\end{table}
%

\subsection{Classical Logic}
We assume familiarity with classical logic. A logical language for classical logic $\mL$ is a set of well-formed formulas. Let us define $\cn_c: 2^{\mL_c} \to 2^{\mL_c}$ as follows:  
For $X \subseteq \mL_c$, a formula $x \in \mL_c$ satisfies $x \in \cn_c(X)$ iff the inference rule $\frac{y}{x}$ is applied to $X$ such that $y \in X$. Define $\cnb_c(X) = \bigcup_{n \geq 0}\cn_{c}^{n}(X)$. In particular, $\cnb_c$ is a consequence operator satisfying $A_1 -\ A_5$. Examples of classical logic are propositional logic and first-order logic. Next, we will consider propositional logic as being given by our abstract notions. Since first-order logic can be similarly simulated, we do not consider here in detail.

We next present \emph{propositional logic} as a special case of classical logic.
Let $A$  be a set of propositional atoms. Any atoms $a \in A$ is a well-formed formula wrt. $A$. If $\phi$ and $\alpha$ are well-formed formulas wrt. $A$ then $\neg \phi$, $\phi \wedge \alpha$, $\phi \vee \alpha$ are well-formulas wrt. $A$ (we also assume that the usual abbreviations $\supset$, $\leftrightarrow$ are defined accordingly). Then $\mL_{p}$ is the set of well-formed formulas wrt. $A$.

Let~$\cn_{p}: 2^ {\mL_{p}} \to 2^ {\mL_p}$ be defined as follows: for $X \subseteq \mL_p$, an element~$x\in \mL_p$
satisfies $x\in \cn_p(X)$ iff there are~$y_1,\ldots,y_j \in X$
such that~$x$ can be obtained from~$y_1,\ldots,y_j$ by the application of a single inference rule of propositional logic.

\begin{example} 
\label{ex:pro-logic}
Consider the propositional atoms $A_1 =\{x, y\}$ and the knowledge base $\mK_1 = \{x, y, x \supset \neg y\} \subseteq \mL_p$. Consider a set $\{x, x \supset \neg y\} \subseteq \mK_1$ .
If the inference rule (modus ponens) $\frac{ A, A \supset B}{B}$ is applied to this set, then $\cn_{p}(\{x, x \supset \neg y\}) = \{x, x \supset \neg y, \neg y\}$.
\end{example}
Consider $\cnb_p(X) = \bigcup_{n \geq 0}\cn_{p}^{n}(X)$.  For instance, $\cnb_{p}(\mK_1) = \{x, y, x \supset \neg y, \neg y, x \wedge y, \ldots\}$.
Since propositional logic is coherent and complete, then~$x\in \cnb_p(X) = \{x \mid X \models x \}$ where $\models$ is the entailment relation, i.e., $\phi \models \alpha$ if all models of $\phi$ are models of $\alpha$ in the propositional semantics. In particular,~$\cnb_{p}$ is a consequence operator  satisfying $A_1 -\ A_5$. The propositional logic can be defined as $(\mL_{p} , \cn_{p})$. 

It follows immediately 

\begin{lemma}
    $(\mL_{p} , \cn_{p})$ is an abstract logic.
\end{lemma}

\begin{example} [Continue Example~\ref{ex:pro-logic}] Recall $\mK_1$. The KB admits a MCS: $\{x, y, x \supset \neg y \}$.
    
\end{example}


\subsection{Defeasible Logic}
\label{subsec:defeasible-logic}

Let $(\mL_d, \cn_d)$ be a defeasible logic such as used in defeasible logic programming~\cite{Alejandro2014}, assumption-based argumentation (ABA)~\cite{Dung2009}, ASPIC/ ASPIC+ systems~\cite{Prakken2002,ModgilP14}.
The language for defeasible logic $\mL_d$ includes a set of (strict and defeasible) rules and a set of literals. The rules is the form of $x_1, \ldots x_i \rightarrow_{s} x_{i+1}$ ($x_1, \ldots x_i \rightarrow_{d} x_{i+1}$) where $x_1, \ldots x_i, x_{i+1}$ are literals and $\rightarrow_{s}$ (denote strict rules) and $\rightarrow_{d}$ (denotes defeasible rules) are implication symbols.
\begin{definition}
Define $\cn_d: 2^{\mL_d} \to 2^{\mL_d}$ as follows: for $X \subseteq \mL_d$, a formula $x \in \mL_d$ satisfies $x \in \cn_d(X)$ iff at least of the following properties is true:
\begin{enumerate}
    \item $x$ is a literal in $X$,
    \item there is $(y_1,\ldots,y_j)\rightarrow_s x \in X$, or $(y_1,\ldots,y_j)\rightarrow_d x \in X$ st. ~$\{ y_1,\ldots,y_j \} \subseteq X$.
\end{enumerate}
Define $\cnb_d(X) = \bigcup_{n \geq 0}\cn_{d}^{n}(X)$.
\end{definition}

\begin{remark}  One can describe $\cnb_d$ explicitly.
We have $x \in \cnb_{d}(X)$ iff there exists a finite \emph{sequence} of literals $x_1, \ldots, x_n$ such that
\begin{enumerate}
    \item $x$ is $x_n$, and
    \item for each $x_i \in \{x_1, \ldots, x_n \}$,
    \begin{itemize}
        \item  there is $ y_1 , \ldots , y_j \rightarrow_{s} x_i \in X$, or $ y_1 , \ldots , y_j \rightarrow_{d} x_i \in X$, such that $\{ y_1 , \ldots , y_j \} \subseteq \{x_1, \ldots , x_{i-1} \}$,
        \item or $x_i$ is a literal in $X$.
    \end{itemize}
\end{enumerate}

Note that if $x \in \cn_{d}^{n}(X)$, the above sequence $x_1, \ldots, x_n$ might have length~$m\neq n$: intuitively~$n$ is the depth of the proof tree while~$m$ is the number of nodes.
\end{remark}







\begin{example}
\label{ex:de-logic}
Consider the KB $\mK_2 = \{x, x \rightarrow_s y, t \rightarrow_d z \} \subseteq \mL_d$. $\cnb_{d}(\mK_2) = \{x, y\}$ where the sequence of literals in the derivation is $x, y$. The KB admits a MCS: $\{x, x \rightarrow_s y, t \rightarrow_d z \}$
\end{example}


\begin{remark}
\label{re:aspic}
For ASPIC/ASPIC+ systems~\cite{Prakken2002,ModgilP14}, Prakken claimed that strict and defeasible rules can be considered in two ways: (1) they encode information of the knowledge base, in which case they are part of the logical language $\mL_d$, (2) they represent inference rules, in which case they are part of the consequence operator.
These ways can encoded by consequence operators as in~\cite{Amgoud12}. Our definition of $\cnb_{d}$ can align with the later interpretation as done in~\cite{Amgoud12}.
In particular, if we consider $X$ being a set of literals of $\mL_d$ instead of being a set of literals and rules as above, the definitions of $\cn_d$ and $\cnb_d$ still hold for this case. 
Thus, the defeasible logic of ASPIC/ASPIC+ can be represented by the logic $(\mL_d, \cn_d)$ in our settings.
\end{remark}

\begin{proposition} [~\cite{Amgoud12}]
    $\cnb_d$ satisfies $A_1 -\ A_3$.
\end{proposition}

It follows immediately

\begin{lemma}
    $(\mL_d, \cn_d)$ is an abstract logic.
\end{lemma}

In~\cite{CaminadaA07}, proposals for argumentation using defeasible logic were criticized for violating the postulates that they proposed for acceptable argumentation. One solution is to introduce \emph{contraposition} into the reasoning of the underlying logic. 
This solution can be seen as another representation of defeasible logic.
We introduce contraposition by defining a consequence operator as follows: 

Consider $\mL_{co}$ containing a set of literal and a set of (strict and defeasible) rules $\mR_{s}$ ($\mR_{d})$. For this case represent inference rules, namely, they are part of a consequence operator.  For $\Delta \subseteq \mL_{co}$,  $\texttt{Contrapositives}(\Delta)$ is the set of contrapositives formed from the rules in $\Delta$. For instance, a strict rule $s$ is a contraposition of the rule $\phi_1, \ldots, \phi_n \rightarrow_{s} \alpha \in \mR_s$ iff $s = \phi_1, \ldots, \phi_{i-1}, \neg \alpha, \phi_{i+1}, \ldots, \phi_n \rightarrow_s \neg \phi_i$ for $ 1 \leq i \leq n$.

\begin{definition}
\label{def:cn-contrapositive}
Define $\cn_{co} : 2^{\mL_{co}} \to 2^{\mL_{co}}$ as follows: for a set of literals $X \subseteq \mL_{co}$, a formula $x \in \mL_{co}$ satisfies $x \in \cn_{co}(X)$ iff at least of the following properties is true:
\begin{enumerate}
    \item $x$ is a literal in $X$,
    \item there is $(y_1,\ldots,y_j)\rightarrow_s x \in \mR_{s} \cup \texttt{ \textup{Contrapositives}}(\mR_s)$, or $(y_1,\ldots,y_j)\rightarrow_d x \in \mR_{d}  \cup \texttt{\textup{Contrapositives}}(\mR_d)$ such that~$\{ y_1,\ldots,y_j \} \subseteq X$.
\end{enumerate}
Define $\cnb_{co}(X) = \bigcup_{n \geq 0}\cn_{co}^{n}(X)$.
\end{definition}

\begin{remark}
Similarly, one can represent $\cnb_{co}$ as follows: $x \in \cnb_{co}(X)$ iff there exists a \emph{sequence} of literals $x_1, \ldots, x_n$ such that
\begin{enumerate}
    \item $x$ is $x_n$, and
    \item for each $x_i \in \{x_1, \ldots, x_n \}$,
    \begin{itemize}
        \item  there is $ y_1 , \ldots , y_j \rightarrow_{s} x_i \in \mR_s \cup \texttt{\textup{Contrapositives}}(\mR_s)$, or $ y_1 , \ldots , y_j \rightarrow_{d} x_i \in \mR_d \cup \texttt{\textup{Contrapositives}}(\mR_d)$, such that $\{ y_1 , \ldots , y_j \} \subseteq \{x_1, \ldots , x_{i-1} \}$,
        \item or $x_i$ is a literal in $X$.
    \end{itemize}
\end{enumerate}
\end{remark}

\begin{proposition} 
    $(\mL_{co}, \cn_{co})$ is an abstract logic. $\cnb_{co}$ satisfies $A_1 -\ A_3$.
\end{proposition}

\begin{example}
    Consider $\mK_3 = \{q, \neg r, p \wedge q \rightarrow_{d} r, \neg p \rightarrow_s u\}$, $\texttt{\textup{Contrapositives}}(\mK_3) = \{\neg r \wedge q \rightarrow_{d} \neg p, \neg r \wedge p \rightarrow_{d} \neg q, \neg u \rightarrow_s p\}$. Then $\cnb_{co}(\mK_3) = \{q, \neg r, \neg p, u\}$  where the sequence of literals in the derivation is $q, \neg r, \neg p, u$. The KB admits MCSs: $\{q, \neg r, \neg p \rightarrow_s u \}$ and $\{q, p \wedge q \rightarrow_{d} r, \neg p \rightarrow_s u\} $.
\end{example}

\subsection{\datalogPM}

We consider  \datalogPM~\cite{CALI201257}, and shall use it to illustrate our demonstrations through the paper.

We assume a set \NT of \emph{terms} which contain variables, constants and function terms. An atom is of the form $P(\Vec{t})$, with $P$ a predicate name and $\Vec{t}$ a vector of terms, which is \emph{ground} if it contains no variables. 
A \emph{database} is a finite set of ground atoms (called \emph{facts}).
A \emph{tuple-generating dependency} (TGD) $\sigma$ is a first-order formula of the form $\forall \x \forall \y \phi(\x, \y) \rightarrow \exists \Vec{z} \psi(\x, \Vec{z})$, where $\phi(\x, \y)$ and $\psi(\x, \Vec{z})$ 
are non-empty conjunctions of atoms.
We leave out the universal quantification, and  refer to  $\phi(\x, \y)$ and $\psi(\x, \Vec{z})$ as the \emph{body} ad \emph{head} of $\sigma$.
A \emph{negative constraint} (NC) $\delta$ is a rule of the form $\forall \x$ $\phi(\x)\rightarrow \bot$
where $\phi(\x)$ is a conjunction of atoms. We may leave out the universal restriction.
A language for \datalogPM $\mL_{da}$ includes a set of facts and a set of TGDs and NCs. A knowledge base $\mK$ of  $\mL_{da}$ is now a tuple $(\mF, \mR, \mC)$ where a database $\mF$, a set $\mR$ of TGDs and a set $\mC$ of NCs.

Define $\cn_{da} : 2^{\mL_{da}} \to 2^{\mL_{da}}$ as follows: Let $X$ be a set of facts of $\mL_{da}$, an element $x \in \mL_{da}$ satisfies $x \in \cn_{da}(X)$ iff there are $y_1, \ldots, y_j \in X$ s.t. $x$ can be obtained from $y_1, \ldots, y_j$ by the application of a single inference rule. Note that we treat such TGDs and NCs as inference rules.

Consider $\cnb_{da}(X) = \bigcup_{n \geq 0}\cn_{da}^{n}(X)$. Similar to proposition logic, ~$x \in \cnb_{da}(X) = \{x \mid X \models x \}$ where $\models$ is the entailment of first-order formulas, i.e., $X \models x$ holds iff every model of all elements in $X$ is also a model of $x$. 
$\cnb_{da}$ satisfies the properties $A_1, A_2$.
Note that the finiteness property ($A_3$) still holds for some fragments of \datalogPM, such as \emph{guarded}, \emph{weakly guarded} \datalogPM. 

It follows immediately

\begin{lemma}
    $(\mL_{da}, \cn_{da})$ is an abstract logic.
\end{lemma}

\begin{example} [Continue Example~\ref{ex:motivation-ex}]
\label{ex:motivation} Recall $\mK_1$.
The KB admits MSCs (called \emph{repairs} in \datalogPM):
\begin{align*}
        & \mB_1 =  \{ \teAs(\vi,\kd), \uc(\kd) \}  \quad 
          \mB_3 = \{\teAs(\vi,\kd), \te(\vi,\kr), \te(\vi, \kd), \gc(\kr) \}\\
        & \mB_2 = \{\te(\vi,\kr), \gc(\kr), \te(\vi, \kd) \}  \quad
        \mB_4 = \{\teAs(\vi,\kd), \te(\vi,\kr), \te(\vi, \kd), \uc(\kd) \} \\
       & \mB_5 = \{\gc(\kr) , \te(\vi,\kr), \te(\vi, \kd), \uc(\kd) \}
         \quad
        \mB_6 = \{\uc(\kd), \teAs(\vi, \kd), \gc(\kr) \}
    \end{align*}
Consider $q_1 = \rese(\vi)$. We have that $\vi$ is a \emph{possible answer} for $q_1$ since $q_1$ is entailed in some repairs, such as $\mB_2,\ \mB_3,\ \mB_5$.
\end{example}

\section{Proof-oriented (Logical) Argumentations}
\label{sec:proof-arg}
In this section, we present \emph{proof-oriented (logical) argumentations} (P-SAFs) and their ingredients. We also provide insights into the connections between our framework and state-of-the-art argumentation frameworks. We then show the close relations of reasoning with P-SAFs to
reasoning with MCSs. 

\subsection{Arguments, Collective Attacks and Proof-oriented Argumentations}

 \textit{Logical arguments} (\textit{arguments} for short) built from a KB may be defined in different ways. For instance, arguments are represented by the notion of \emph{sequents}~\cite{ArieliS19}, \emph{proof}~\cite{SCHULZ_TONI_2016,Dung2009}, a pair of $(\Gamma,\ \psi)$ where $\Gamma$ is the
\textit{support}, or \textit{premises}, or \textit{assumptions} of the argument, and $\psi$ is the \textit{claim}, or \textit{conclusion}, of the argument~\cite{LoanHo2022,ARIOUA2017244}. To improve explanations in terms of representation and understanding, we choose the form of \emph{proof} to represent arguments. 
The proof is in the form of a tree.

\begin{definition}
\label{def:ab-arg}
A formula $\phi \in \mL$ is \emph{tree-derivable} from a set of \emph{fact-premises} $H \subseteq \mK$
if there is a tree such that
\begin{itemize}
    \item the root holds $\phi$;
    \item $H$ is the set of formulas held by leaves;
    \item for every inner node $N$, if $N$ holds the formula $\beta_0$, then its successors hold $n$ formulas $\beta_1, \ldots , \beta_n$ such that $\beta_0 \in \cn(\{\beta_1, \ldots , \beta_n\})$.
\end{itemize}
If such a tree exists (it might not be unique), we call $A : H \Rightarrow \phi$
an \emph{argument} with the \emph{support set} $\Sup(A) = H$ and the \emph{conclusion} $\Con(A) = \phi$. We denote the set of arguments induced from $ \mK$ by $\Arg_{\mK}$.


\end{definition}

\begin{remark}
 By Definition~\ref{def:ab-arg} it follows that $ H \Rightarrow \phi$ is an argument iff $\phi \in \cnb(H)$.   
\end{remark}

Note that an individual argument can be represented by several different trees (with the same root and leaves). We assume these trees represent the same arguments; otherwise, we could have infinitely many arguments with the same support set and conclusion.

Intuitively, a tree represents a possible derivation of the formula at its root and the fact-premise made at its leaves.
The leaves of the tree, constituting the fact-premise, belong to $H = \cn^0 (H)$.
If a node $\beta$ has children nodes
$\beta_{a_1} \in \cn^{i_1} (H)$, \ldots, $\beta_{a_k} \in \cn^{i_k} (H)$,
then $\beta \in \cn^{i+1}(X)$ where $i=\max\{i_1,\ldots,i_k\}$ because by the extension property $\cn^{i_1} (H),\ldots,\cn^{i_k} (H)\subseteq \cn^{i}(H)$.
The root $\phi$, constituting the conclusion, belongs to $\cn^n(H)$,
where $n$ is the longest path from leaf to root.
Note that, by the extension property,
if $\beta \in \cn^i(H)$, then also $\beta \in \cn^{i+1}(H)$, $\beta \in \cn^{i+2}(H)$, \ldots.
The idea is to have $i$ in $\cn^i(H)$ as small as possible (we don't want to argue
longer than necessary).


Some proposals for logic-based argumentation stipulate additionally that the argument's support is consistent and/or that none of its subsets entails the argument's conclusion (see~\cite{Hunter2010}).
However, such restrictions, i.e., minimality and consistency, are not substantial (although required for some specific logics).
In some proposals, the requirement that the support of an argument is consistent may be irrelevant for some logics, especially when consistency is defined by satisfiability.
For instance, in Priest’s three-valued logic~\cite{Priest89} or Belnap’s four-valued logic~\cite{Belnap1977}, every set of formulas in the language of $\{ \neg, \vee, \wedge\}$ is satisfiable.
In frameworks in which the supports of arguments are represented only by literals (atomic formulas or their negation),  arguments like $A = \{a, b\} \Rightarrow a \vee b$ are excluded since their supports are not minimal, although one may consider$\{a, b\}$ a stronger support for $a \vee b$ than, say, $\{a\}$, 
since the set $\{a, b\}$ logically implies every minimal support of $a \vee b$.
To keep our framework as general as possible, we do not consider the extra restrictions for our definition of arguments (See~\cite{Hunter2010,ArieliS19} for further justifications of this choice).


We present instantiations to show the generality of Definition~\ref{def:ab-arg} for generating arguments in argumentation systems in the literature.

\begin{itemize}
\item We start with \emph{deductive argumentation} that uses classical logic. In~~\cite{BesnardH01}, arguments as pairs of premises and conclusions can be simulated in our settings, and for which $H \Rightarrow \phi$ is an argument (in the form of tree-derivations), where $H \subseteq \mL_c$ and $\phi \in \mL_c$ iff $\phi \in \cnb_{c}(H)$, $H$ is minimal (i.e., there is no $H^{\prime} \subset H$ such that $\phi \in \cnb_{c}(H^{\prime})$) and $H$ is consistent. 
For example, we use the propositional logic in Example~\ref{ex:pro-logic}, and the following is an argument in propositional logic $A : \{x, x \supset \neg y \} \Rightarrow \neg y$. Tree-representation of $A$ is shown in Figure~\ref{fig:propositional-DeLP-sequent} (left).
Similarly, since most Description Logics (DLs), such as $ALC$, DL-Lite families, Horn DL, etc.,  are decidable fragments of first-order logic, it is straightforward to apply Definition~\ref{def:ab-arg} to encode arguments of the framework using the DL $ALC$ in~\cite{ZhangL13}.

\item We consider defeasible logic approaches to argumentation, such as~\cite{Dung2009,DimopoulosD0R0W24,Rapberger2024,Lehtonen2024,Alejandro2014}. 
For \emph{defeasible logic programming}~\cite{Alejandro2014}, $H \Rightarrow \phi$ is an argument (in the form of tree-derivations) iff $\phi \in \cnb_d(H)$ and there is no $H^{\prime} \subset H$ such that $\phi \in \cnb_d(H^{\prime})$ and it is not the case that there is $\alpha$ such that $\alpha \in \cnb_d(H)$ and $\neg \alpha \in \cnb_d(H)$ (i.e. $H$ is a minimal consistent set entailing $\phi$).

For "\emph{flat}"-\emph{ABA}~\cite{Dung2009,DimopoulosD0R0W24}, assume that $\cn_d$ ignores differences between various implication symbols in the knowledge base, and for which $H \Rightarrow \phi$ is an argument iff $\phi \in \cnb_d(H)$ where $H \subseteq \mL_{d}$. In this case, the argument, from the support $H$ to the conclusion $\phi$,
can be described as tree-derivations by $\cn_d$.
Note that the minimality and consistency requirements are dropped. 
Similarly, in \emph{"non-flat"}-\emph{ABAs}~\cite{Rapberger2024,Lehtonen2024}, arguments as tree-derivations can be simulated in our setting.

Note, in~\cite{Alejandro2014} only the defeasible rules are represented in the support of the argument, and in~~\cite{Dung2009,DimopoulosD0R0W24,Rapberger2024,Lehtonen2024} only the literals are represented in the support of the argument, but in both cases it is a trivial change (as we do here) to represent both the rules and literals used in the derivation in the support of the argument.

\begin{example}
For $\mK_5 = \{a, \neg b, a \rightarrow_s \neg c ,\ \neg b \wedge \neg c \rightarrow_{d} s,\ s \rightarrow_{s} t,\ a \wedge t \rightarrow_{d} u \}$, the following is an argument in defeasible logic programming $B : \{a, \neg b, a \rightarrow_s \neg c , \neg b \wedge \neg c \rightarrow_{d} s \} \Rightarrow s$ with the sequences of literals $a, \neg c , \neg b, s$. Tree-representations of the arguments are shown in Figure~\ref{fig:propositional-DeLP-sequent} (middle).

For $\mK_6 = \{p, \neg q, s, p \rightarrow \neg r, \neg q \wedge \neg r \wedge s \rightarrow t, t \wedge p \rightarrow u, v \}$, the following is an argument in ABA $C : \{p, \neg q, s, p \rightarrow \neg r, \neg q \wedge \neg r \wedge s \rightarrow t \} \Rightarrow t$.

\end{example}

\item We translate ASPIC/ ASPIC+~\cite{Prakken2002,ModgilP14} into our work as follows:

We have considered the underlying logic of ASPIC/ ASPIC+ as being given by $\cnb_{d}$ (see Remark~\ref{re:aspic}) and $\mL_d$ including the set of literals and strict/ defeasible rules.

We recall argument of the form $A_1, \ldots, A_n \rightarrow_s / \rightarrow_d \phi$ 
in these systems as follows:
\begin{enumerate}
    \item Rules of the form $\rightarrow_s / \rightarrow_d \alpha$ , are arguments with conclusion $\alpha$.

    \item  Let $r$ be a strict/defeasible rule of the form $\beta_1, \ldots, \beta_n \rightarrow_s / \rightarrow_d \phi$, $n \geq 0$.
    Further suppose that $A_1, \ldots, A_n$, $n \geq 0$, are arguments with conclusions $\beta_1, \ldots, \beta_n$ respectively.
    Then $A_1, \ldots, A_n \rightarrow_s / \rightarrow_d \phi$ is an argument with conclusion $\phi$ and last rule $r$.
    
    \item Every argument is constructed by applying finitely many times the above two steps.
\end{enumerate}

The arguments of the form $A_1, \ldots, A_n \rightarrow_s / \rightarrow_d \phi$ can be viewed as tree-derivations in the sense of Definition~\ref{def:ab-arg}, in which the conclusion of the argument is $\phi$; 
the support $H$ of the argument is the set of leaves that are rules of the form $\rightarrow_s / \rightarrow_d \alpha_i$ such that $\alpha_i \in \cn^{0}_{d}(H)$. In this view, the root of the tree is labelled by $\phi$ such $\phi \in \cn^{n}_{d}(H)$; the children $\beta_i$, $i = 1, \ldots, n$, of the root are the roots of subtrees $A_1, \ldots, A_n$; if $\phi \in \cn^{n}_{d}(H)$, then $\beta_i \in \cn^{n-1}_{d}(H)$.
Since $\cnb_{d}(H) = \bigcup_{n}\cn^{n}_{d}(H) $, it follows that $\phi \in \cnb_{d}(H)$.
Note that if $n = 0$, the tree consists of just the root that is the rule of the form $\rightarrow_s / \rightarrow_d \phi$.


\item In argumentation framework for \datalogPM~\cite{ARIOUA201776,Amgoud12}, arguments, viewed as pairs of the premises $H$ (i.e., the set of facts) and the conclusion $\phi$ (i.e., the derived fact), can be represented as tree-derivations in our definitions as follows:  For a consistent set $H \subseteq \mF$ and $\phi \in \mL_{da}$, $H \Rightarrow \phi$ is an argument in the sense of Definition~\ref{def:ab-arg} iff $\phi \in \cnb_{da}(H)$, in which $\phi$ is the root of the tree; $H$ are the leaves.

\begin{example}
    Let us continue Example~\ref{ex:motivation}, the following is an argument in the framework using \datalogPM~$A_7 : \{ \teAs(\vi, \kd),\ \uc(\kd) \} \Rightarrow \ta(\vi)$. By Definition~\ref{def:ab-arg}, the argument can be viewed as a proof tree with the root labelled by $\ta(\vi)$ and the leaves labelled $\teAs(\vi, \kd),\ \uc(\kd)$.
\end{example}
\end{itemize}

\begin{figure}
\begin{tikzpicture}
    \node (dt) at (0,0) {\includegraphics[scale=0.7]{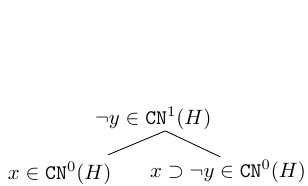}};
    \node (d1) at (5, 0) {\includegraphics[scale=0.7]{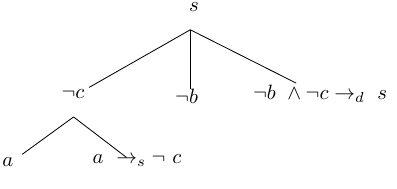}};
     \node (d1) at (10, 0) {\includegraphics[scale=0.7]{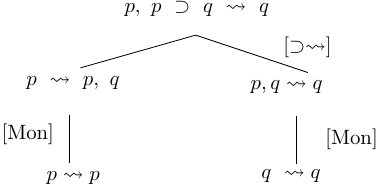}};
\end{tikzpicture}
\caption{
Tree-representation for arguments wrt logics.
}
\label{fig:propositional-DeLP-sequent}
\end{figure}


As shown in examples of~\cite{Yun2020SetsOA,Amgoud12,ArieliH24}, binary attacks, used in the literature~\cite{ArieliS19,ARIOUA201776,LoanHo2022,ZhangL13,Castagna21}, are not enough expressive to capture cases in which n-ary conflicts may arise.
To overcome this limit, some argumentation frameworks introduced the notion of \emph{collective attacks} to better capture non-binary conflicts, and so improve the decision making process in various conflicting situations. To ensure the generality of our framework, we introduce collective attacks.

\begin{definition} [Collective Attacks]
\label{def:ab-att} 
Let $A: \Gamma \Rightarrow \alpha$ be an argument and $\mX \subseteq \Arg_{\mK}$ be a set of arguments such that $\bigcup_{X \in \mX}\Sup(X)$ is consistent. We say that
    \begin{itemize}
        \item $\mX$ \emph{undercut-attacks} $A$ iff there is $\Gamma^{\prime} \subseteq \Gamma$ s.t $\bigcup_{X \in \mX} \{ \Con(X)\} \cup \Gamma^{\prime}$ is inconsistent.
        \item $\mX$ \emph{rebuttal-attacks} $A$ iff $\bigcup_{X \in \mX} \{ \Con(X)\} \cup \{\alpha \}$ is inconsistent.
  
    \end{itemize}
We can say that $\mX$ \emph{attacks} $A$ for short. We use $\Att_{\mK} \subseteq  2^{\Arg_{\mK}} \times \Arg_{\mK}$ to denote \emph{the set of attacks} induced from $ \mK$.
\end{definition}







Note that deductive argumentation can capture n-ary conflicts. However, as discussed in~\cite{Yun2020SetsOA,yun2018}, it argued that the argumentation framework using \datalogPM, an instance of deductive argumentation, may generate a large number of arguments and attacks when using the definition of deductive arguments, as in~\cite{ARIOUA201776}. To address this problem, some redundant arguments are dropped, as discussed in~\cite{yun2018}, or arguments are re-defined as those in ASPIC+, as seen in~\cite{Yun2020SetsOA}. Then the attack relation must be redesigned to preserve all conflicts. In particular, n-ary attacks are allowed where arguments can jointly attack an argument. We will show this issue in the following example.  


\begin{example}
Consider $\mK_2 = ( \mF_2, \mR_2, \mC_2)$ where
\begin{align*}
	\mR_2  = & \emptyset,\\
	\mC_2  = & \{ A(x) \land B(x) \land C(x) \rightarrow \bot \}, \\
    \mF_2 =  &\{A(a), B(a), C(a) \}.
   \end{align*} 
The deductive argumentation approach~\cite{ARIOUA201776} uses six arguments
\begin{align*}
& C_2: ( \{B(a) \}, B(a)), C_3: (\{ C(a) \}, C(a) ), C_4: ( \{A(a), B(a) \}, A(a) \land B(a) ), \\
& C_1: ( \{A(a)\}, A(a)), C_5: ( \{ A(a), C(a) \}, A(a) \land C(a)), C_6: ( \{B(a), C(a) \}, B(a) \land C(a))    
\end{align*}
to obtains the preferred extensions: $\{C_1, C_2, C_4\}$,
$\{C_1, C_3, C_5\}$, $\{C_2, C_3, C_6\}$.

In contrast, our approach uses three arguments $B_1 : \{ A(a)\} \Rightarrow A(a)$, $B_2 : \{ B(a)\} \Rightarrow B(a)$, $B_3 : \{ C(a)\} \Rightarrow C(a)$  with collective attacks, such as $\{B_1, B_2 \}$ attacks $B_3$, etc., to obtain extensions $\{B_1, B_2\}, \{B_1, B_3\}, \{B_2, B_3\}$.

\end{example}

\begin{remark}
\label{re:compare-attacks}
Similar to structured argumentation, such as deductive argumentation for propositional logic~\cite{BesnardH01}, DLs~\cite{ZhangL13}, \datalogPM~\cite{ARIOUA201776},
DeLP systems~\cite{Alejandro2014},  ASPIC systems~\cite{Prakken2002}
and sequent-based argumentation~\cite{HEYNINCK2020103,ArieliH24}, attacks in our framework are defined between individual arguments.
In contrast, in ABA systems~\cite{Dung2009,DimopoulosD0R0W24,Rapberger2024,Lehtonen2024}, attacks are defined between sets of assumptions.
However, in these ABA systems, the arguments generated from a set of assumptions are \emph{tree-derivations} (both notions are used interchangeably), which can be instantiated by Definition~\ref{def:ab-arg}, see above.
Thus, the attacks defined on assumptions are equivalent to the attacks defined on the level of arguments.
\end{remark}

\begin{remark}
    Note that the definition of collective attacks holds if we only consider ASPIC+ without preferences~\cite{ModgilP14}. We leave the case of preferences for future work.
\end{remark}

We introduce \emph{proof-oriented argumentation} (P-SAF) as an instantiation of SAFs~\cite{Nielsen2007}. Our framework is comparable to the one of~\cite{loanho_2024} in that both are applied to abstract logic. However, arguments in our setting differ from those in~\cite{loanho_2024} in that we represent arguments in the form of a tree.

\begin{definition} \label{def:ab-af} Let $ \mK$ be a KB, the corresponding \emph{proof-oriented (logical) argumentation (P-SAF)} $\mAF_{\mK}$ is the pair $(\Arg_{\mK}, \Att_{\mK})$ where $\Arg_{\mK}$ is the set of arguments induced from $ \mK$ and $\Att_{\mK}$ is the set of attacks.
\end{definition}


In the next subsections, we show that the existing argumentation frameworks are instances of logic-associated argumentation frameworks.

\subsection{Translating the Existing Argumentation Frameworks to P-SAFs}
\label{subsec:relation-framework}

We have already shown that the existing frameworks (deductive argumentation~\cite{BesnardH01,ZhangL13,ARIOUA201776}, DeLP systems~\cite{Alejandro2014},  ASPIC systems~\cite{Prakken2002}, ASPIC+ without preferences~\cite{ModgilP14}, ABA systems~\cite{Dung2009,DimopoulosD0R0W24,Rapberger2024,Lehtonen2024}) can be seen as instances of our settings. Now we show how sequent-based argumentation~\cite{ArieliS19,BorgAS17} and contrapositive ABAs~\cite{HEYNINCK2020103,ArieliH24} fit in our framework.

\begin{itemize}

\item Sequent-based argumentation~\cite{ArieliS19}, using propositional logic, represents arguments as \emph{sequents}.
The construction of arguments from simpler arguments is done by the inference rules of the \emph{sequent calculus}.
\emph{Attack rules} are represented as \emph{sequent elimination rules}. 
The ingredients of sequent-based argumentation may be simulated in our setting:

We start with a logic $(\mL_s, \cn_s)$. $\mL_s$ is a propositional language having a set of atomic formulas $\AT(\mL_s)$.  
If $\phi$ and $\alpha$ are formulas wrt.$\AT(\mL_s)$ then $\neg \phi$, $\phi \wedge \alpha$ are formulas wrt. $\AT(\mL_s)$. 
We assume that the implication $\supset$ and $\leftrightarrow$ are defined accordingly.
Propositional logic can be modelled by using sequents~\cite{ArieliS19}.
A sequent is a formula in the language $\mL_s$ of propositional logic enriched by the addition of a new symbol $\leadsto$.
We call such sequent the \emph{s-formula} of $\mL_s$ to avoid ambiguity.
In particular, for a formula $p \in \mL_s$ the \emph{axiom} $p \leadsto p$ are a s-formula in $\mL_s$. 
In general, for any set of formulas $\Psi \subseteq \mL_s$ and $\phi \in \mL_s$,  the sequents $\Psi \leadsto \phi$ are s-formulas of $\mL_s$. 


Define $\cn_s$ as follows: 
For a set of formulas $X \subseteq \mL_s$, a formula $\phi \in \mL_s$ satisfies $\phi \in \cn_s(X)$ iff an inference rule
$\frac{\Psi_1 \leadsto \phi_1 \ldots \Psi_n \leadsto \phi_n}{\Psi \leadsto \phi}$, where the sequents $\Psi \leadsto \phi$ and $\Psi_i \leadsto \phi_i$ ($i = 1, \ldots, n$) are s-formulas of $\mL_s$, is applied to $X$ such that $\Psi_1, \ldots, \Psi_n$ are subsets of $X$.
We here consider the inference rules as \emph{structural rules} and \emph{logical rules} in~\cite{ArieliS19}.
Then we define $\cnb_s(X) = \bigcup_{n \geq 0}\cn_{s}^{n}(X)$.


Let us define arguments in the sense of Definition~\ref{def:ab-arg}: For a set of formulas $H \subseteq \mL_s$, $H \Rightarrow \phi$ is an argument iff $\phi \in \cnb_s(H)$. In this case, the argument, from the premise $H$ to the conclusion $\phi$, can be described by a sequence of applications of inference rules.
Such sequence is naturally organized in the shape of a tree by $\cn_s$.  Each step up the tree corresponds to an application of an inference rule. The root of the tree is the final sequent (the conclusion), and the leaves are the axioms or initial sequents.


We show how the attack rules can be described in terms of corresponding attack relations in Definition~\ref{def:ab-att}. The attack rule has the form of 
$\frac{\Psi_1 \leadsto \phi_1 , \ldots ,\Psi_n \leadsto \phi_n}{\Psi_n \not \leadsto \phi_n}$, in which
the first sequent in the attack rule’s prerequisites is the “attacking” sequent, the last sequent in the attack rule’s prerequisites is the “attacked” sequent, 
and the other prerequisites are the conditions for the attack. According to the discussion above, these sequents $\Psi_i \leadsto \phi_i$, ($i = 1, \ldots, n$), can be viewed as arguments $A_i : \Psi_i \Rightarrow \phi_i$ in the sense of Definition~\ref{def:ab-arg} where $\phi_i \in \cnb_s (\Psi_i)$. Then, in this view, the first sequent $\Psi_1 \leadsto \phi_1$ is the attacking argument $A_1$, the last sequent $\Psi_n \leadsto \phi_n$ is the attacked argument $A_n$, and the conclusions of the attack rule are the eliminations of the attacked arguments, meaning that $A_n$ is removed since $A_1$ attacks $A_n$ in the sense of Definition~\ref{def:ab-att}.

\begin{example} [Continue Example~\ref{ex:pro-logic}] Consider $\mK = \{x, x \supset y, \neg y \} \subseteq \mL_s$.
The following is an argument in propositional logic $A : \{x, x \supset y \} \Rightarrow  y$, $B : \{ \neg y \} \Rightarrow \neg y$. $A$ attacks $B$ since $\{ x, x \supset y, \neg y \}$ is inconsistent, i.e., $\cnb_s( \{ x, x \supset y, \neg y \}) = \mL_s$. Tree-representations of the arguments are shown in Figure~\ref{fig:propositional-DeLP-sequent}(Right), in which $[Mon]$ and $[\supset,\leadsto]$ are the names of inference rules.
\end{example}

\item  Contrapositive ABA ~\cite{HEYNINCK2020103,ArieliH24} may be based on propositional logic and \emph{strict} and \emph{candidate (defeasible) assumptions} consists of \emph{arbitrary} formulas in the language of that logic. Attacks are defined between sets of assumptions, i.e., defeasible assumptions may be attacked in the presence of a counter defeasible information.
Our P-SAF framework using logic $(\mL_{co}, \cn_{co})$ can simulate contrapositive ABAs as follows:
 
Assume that an implication connective $\supset$ is deductive (i.e., it is a $\vdash$-implication in contrapositive ABAs) and converting such implications $\supset$ (i.e., $\phi_1 \wedge \cdots \wedge \phi_n \supset \psi$) to rules of the form $\phi_1, \ldots, \phi_n \rightarrow \psi$ in $\mL_{co}$. Here we ignore the distinction between defeasible and strict rules.
With this assumption, the rules in $\mL_{co}$ can be treated as $\vdash$-\emph{implication}, i.e.,$\{ \phi_1, \ldots, \phi_n \rightarrow \psi \in \mL_{co} \mid \phi_1, \ldots, \phi_n \vdash \psi \}$; their contrapositions treated as $\vdash$-\emph{contrapositive}, i.e., $\{ \phi_1, \ldots, \phi_{i-1}, \neg \psi, \phi_{i+1}, \dots \phi_n \rightarrow \neg \phi_i \mid \phi_1, \ldots, \phi_{i-1}, \neg \psi, \phi_{i+1}, \dots \phi_n \vdash \neg \phi_i \}$ \footnote{ See definitions of $\vdash$-implication and $\vdash$-contrapositive in~~\cite{HEYNINCK2020103,ArieliH24}}.
This translation views the contrapositive ABA as a special case of the traditional definition of ABA~\cite{Dung2009}; also the traditional ABA can be simulated in our P-SAF using $(\mL_{co}, \cn_{co})$. Thus, the results and concepts of P-SAFs can apply to the contrapositive ABAs.
Indeed, first, $\mL_{co}$ includes the strict and candidate assumptions
~\footnote{We abuse the term “strict and candidate assumptions” and refer them as "literals".}
and the set of rules.
These rules as reasoning patterns are used in $\cn_{co}$ as defined in Definition~\ref{def:cn-contrapositive}.
Second, by Definition~\ref{def:ab-arg}, for $H \subseteq \mL_{co}$ be a set of assumptions and $\phi \in \mL_{co}$ , $H \Rightarrow \phi$ is an argument iff $\phi \in \cnb_{co}(H)$. Third, the attacks defined on assumptions in traditional ABA are equivalent to those in our S-PAFs (see Remark~\ref{re:compare-attacks} for further explanation).

Note that contrapositive ABAs in~\cite{ArieliH24} (with collective attacks) are analogous to those in~\cite{HEYNINCK2020103}, except they drop the requirement that any set of candidate assumptions contributing to the attacks must be \emph{close}. Similarly, our P-SAF framework, which uses $\cnb_{co}$ in the definition of inconsistency for our attacks,
does not impose this additional requirement.
\end{itemize}

\subsection{Acceptability of P-SAFs and Relations to Reasoning with MSCs}

Semantics of P-SAFs are now defined as in the definition of semantics for SAFs~\cite{Nielsen2007}. These semantics consist of \textit{admissible}, \textit{complete}, \textit{stable}, \textit{preferred} and \textit{grounded semantics}. 

Given a P-SAF $\mAF_{\mK} = (\Arg_{\mK}, \Att_{\mK})$ and $\mS\subseteq \Arg_{\mK}$. $\mS$ \emph{attacks} $\mX$ iff $ \exists A \in \mX$ s.t. $\mS$ attacks $A$.  $\mS$ \emph{defends} $A$ if for each $\mX \subseteq \Arg_{\mK}$ s.t. $\mX$ attacks $A$, some $\mS^{\prime} \subseteq \mS$ attacks $\mX$. An \textit{extension} $\mS$ is called

\begin{itemize}

    \item \emph{conflict-free} if it does not attack itself;

    \item \emph{admissible} $(\adm)$ if it is conflict-free and defends itself.

    \item \emph{complete} $(\cmp)$  if it is  admissible containing all arguments  it defends.

    \item \emph{preferred} $(\prf)$ if it is an inclusion-maximal admissible extension. 

    \item \emph{stable} $(\stb)$ if it is conflict-free and attacks every argument not in it.

    \item \emph{grounded} $(\grd)$ if it is an inclusion-minimal complete extension.
    
\end{itemize}
Note that this implies that each grounded or preferred extension of a P-SAF is an admissible one, the grounded extension is contained in all other extensions.

Let $\Exts_{\sem}( \mAF_{\mK})$ denote \emph{the set of all extensions} of $ \mAF_{\mK}$ under the semantics $\sem \in \{\adm,\ \stb,\ \prf,\ \grd \}$.
Let us define \emph{acceptability} in P-SAFs.
\begin{definition}
  \label{def:accept}
  Let $\mAF_{\mK}$ be the corresponding P-SAF of a KB $\mK$ and $\sem \in \{\adm, \stb, \prf \}$. A formula $\phi \in \mL$ is 
   \begin{itemize}

    \item \emph{credulously} accepted under $\sem$ iff for some~$\mE \in \Exts_{\sem}( \mAF_{\mK})$, $\phi \in \Cons(\mE)$.

    \item \emph{groundedly} accepted under $\grd$ iff for some $\mE \in \Exts_{\grd}( \mAF_{\mK})$, $\phi \in \Cons(\mE)$.

    \item \emph{sceptically} accepted under $\sem$ iff for all $\mE \in \Exts_{\sem}( \mAF_{\mK})$, $\phi \in \Cons(\mE)$.
  \end{itemize}
  \end{definition}

Next, we show the relation to reasoning with maximal consistent subsets in inconsistent KBs.
Proposition~\ref{pro:psaf-link} shows a relation between extensions of P-SAFs and MSCs of KBs.

\begin{proposition}
\label{pro:psaf-link}
 Let $\mAF_{\mK}$ be the corresponding P-SAF of a KB $\mK$. Then,

 \begin{itemize}
     \item the maximal consistent subset of $ \mK$ coincides with the stable/ preferred extension of $\mAF_{\mK}$;

     \item the intersection of the maximal consistent subsets of $ \mK$ coincides with the grounded extension of $\mAF_{\mK}$.
 \end{itemize}
\end{proposition}
\begin{proof} 
The idea of the proof is to show that every preferred extension is the set of arguments generated from a MCS, that every such set of arguments is a stable extension, and that every stable extension is preferred.
The proof of the second statement follows the lemma saying that if there are no rejected arguments under preferred semantics, then the grounded extension is equal to the intersection of all preferred extensions. By the proof of the first statement, every preferred extension is a maximal consistent subset. Thus the second statement is proved.
\end{proof}
\begin{remark}
In general, the grounded extension is contained in the intersection of all maximal consistent subsets.
\end{remark}
The main result of this section, Theorem \ref{thm:ab-link}, which follows from Proposition \ref{pro:psaf-link} generalises results from previous works. 

\begin{theorem}
\label{thm:ab-link}
Let $\mAF_{\mK}$ be the corresponding P-SAF of a KB $\mK$, $\phi \in \mL$ a formula and $\sem \in \{\adm, \stb, \prf \}$. Then, $\phi$ is entailed in
\begin{itemize}
    \item some maximal consistent subset iff $\phi$ is \emph{credulously accepted} under $\sem$.
       
    \item all maximal consistent subsets iff $\phi$ is \emph{sceptically accepted} under $\sem$.
    
    \item the intersection of all maximal consistent subsets iff $\phi$ is \emph{groundedly accepted} under $\grd$.
   
\end{itemize}
\end{theorem}

To argue the quality of P-SAF, it can be shown that it satisfies the rationality postulates introduced in~~\cite{DAgostinoM18,AmgoudB13}. 

\begin{definition}
Let $\mAF_{\mK}$ be the corresponding P-SAF of a KB $\mK$. Wrt. $\sem \in \{\adm, \stb, \prf, \grd\}$, $\mAF_{\mK}$ is 
    \begin{enumerate} 
	\item \emph{closed under $\cnb$} iff for all $\mE \in \Exts_{\sem}(\mAF_{\mK})$, $\Cons(\mE) = \cnb(\Cons(\mE))$;
	\item \emph{consistent} iff for all $\mE \in \Exts_{\sem}(\mAF_{\mK})$, $\Cons(\mE)$ is consistent;
    \end{enumerate}
\end{definition}

 \begin{proposition}
 \label{pro:postulate} Wrt. to any semantics in $\{\adm ,\stb$, $\prf,\grd\}$, $\mAF_{\mK}$ satisfies consistency, closure.
\end{proposition}
The proof of Proposition~\ref{pro:postulate} is analogous to those of Proposition 2 in~\cite{loanho_2024}. Because of this similarity, they are not included in the appendix.

\begin{example} [Continue Example~\ref{ex:motivation}]
\label{ex:KB-arg}
Recall $\mK_1$. Table~\ref{tab:arg} shows the supports and conclusions of all arguments induced from $\mK$. 
The corresponding P-SAF admits $\stb$ ($\prf$) extensions:
$\Exts_{\stb / \prf}(\mAF_1) = \{\mE_1, \ldots, \mE_6\}$, where
$\mE_1 = \Args(\{ \teAs(\vi,\kd), \uc(\kd)\})$ \footnote{Fix $\mF^{\prime} \subseteq \mF$, $\Args(\mF^{\prime})$ is the set of \emph{arguments generated by} $\mF^{\prime}$} $ = \{A_5, A_6, A_2\}$,
and $\mE_2, \ldots, \mE_6$ are obtained in an analogous way. It can be seen that the extensions correspond to the repairs of the KBs (by Theorem~\ref{thm:ab-link}). 

Reconsider $q_1 = \rese(\vi)$. We have that $q_1$ is credulously accepted under $\stb$ ($\prf$) extensions. In other words, $\vi$ is a \emph{possible answer} for $q_1$.  
\end{example}

\begin{table}\vspace{-6mm} 
\centering
  \caption{Supports and conclusions of arguments}
  \begin{tabular}{|c|l|l|}
    \hline
    \textbf{Argument} & \textbf{$\Sup(A_i)$} & \textbf{$\Con(A_i)$} \\
    \hline
   $A_0$ & $\{\te(\vi, \kr)\}$ & $\te(\vi, \kr)$ \\
   $A_9$ & $\{\gc(\kr)\}$ & $\gc(\kr)$ \\
   $A_7$ & $\{ \gc(\kr), \te(\vi, \kr) \}$ & $\fp(\vi)$ \\
   $A_1$ & $\{ \gc(\kr), \te(\vi, \kr) \}$ & $\rese(\vi)$\\
   $A_4$ & $\{\te(\vi, \kd)\}$ & $\te(\vi, \kd)$ \\
   $A_5$ & $\{\teAs(\vi, \kd)\}$ & $\teAs(\vi, \kd)$\\
   $A_6$ & $\{ \uc(\kd) \}$ & $\uc(\kd)$\\
   $A_2$ & $\{ \teAs(\vi, \kd),\ \uc(\kd) \}$ & $\ta(\vi)$\\
   $A_3$ & $\{\te(\vi, \kd)\}$ & $\lect(\vi)$ \\
   $A_{8}$ & $\{\te(\vi, \kr)\}$ & $\lect(\vi)$ \\
   $A_{10}$ & $\{\te(\vi, \kr)\}$ & $\emp(\vi)$ \\
   $A_{11}$ & $\{\te(\vi, \kd)\}$ & $\emp(\vi)$ \\
   $A_{12}$ & $\{ \gc(\kr), \te(\vi, \kr)\}$ & $\emp(\vi)$ \\
    \hline
  \end{tabular}
  \label{tab:arg} 
\end{table}   

In this section, we have translated from KBs into P-SAFs. Consequently, the acceptance of a formula $\phi$ of $\mL$ corresponds to the acceptance of a set of arguments $\mA$ for $\phi$. When we say "a set of arguments $\mA$ for $\phi$", it means simply that for each argument in $\mA$, its consequence is $\phi$.
We next introduce a novel notion of \textit{explanatory dialogue} ("\emph{dialogue}" for short) viewed as a \emph{dialectical proof procedure} in Section~\ref{sec:model-exp-dia}. Section 5 will show how to use a dialogue model to determine and explain the acceptance of $\phi$ wrt argumentation semantics.

\section{Explanatory Dialogue Models}
\label{sec:model-exp-dia}

Inspired by~\cite{prakken_2006,Prakken05}, we develop a \textit{novel explanatory dialogue model} of P-SAF by examining the dispute process involving the exchange of arguments (represented as formulas in KBs) between two agents. The novel explanatory dialogue model can show how to determine and explain the acceptance of a formula wrt argumentation semantics.



\subsection{Basic Notions}
\textbf{Concepts} of a novel dialogue model for P-SAFs include \textbf{utterances, dialogues} and \textbf{concrete dialogue trees} ("\textbf{dialogue tree}" for short).
In this model, a topic language $\mL_{t}$ is abstract logic $(\mL, \cn)$; dialogues are sequences of utterances between two agents $a_1$ and $a_2$ sharing a common language $\mL_{c}$. Utterances are defined as follows:

\begin{definition} [Utterances]
An \emph{utterance} of agents $a_i,\ i \in \{1,2\}$ has the form $u = (a_i, \TG, \CO, \ID)$, where:
\begin{itemize}
    \item $\ID \in \mathbb{N}$ is the \emph{identifier} of the utterance,

    \item $\TG$ is the \emph{target} of the utterance and we impose that $\TG < \ID$,
    \item $\CO \in \mL_c$ (the \emph{content}) is one of the following forms: Fix $\phi \in \mL$ and $\Delta \subseteq \mL$.
    
    \begin{itemize}
         \item $\cla(\phi)$: The agent asserts that $\phi$ is the case,
        
         \item $\off(\Delta, \phi)$: The agent advances \emph{grounds} $\Delta$ for $\phi$ uttered by the previously advanced utterances such that $\phi \in \cn(\Delta)$,
    
        \item $\cont(\Delta,\ \phi)$: The agent advances the formulas $\Delta$ that are contrary to $\phi$ uttered by the previously advanced utterance,
        \item $\cond(\phi)$: The agent gives up debating and admits that $\phi$ is the case,

         \item $\fa(\phi)$: The agent asserts that $\phi$ is a fact in $\mK$.

         \item $\kappa$: The agent does not have or wants to contribute information at that point in the dialogue.

    \end{itemize}  
\end{itemize}
We denote by $\mU$ the set of all utterances. 
\end{definition}

To determine which utterances agents can make to construct a dialogue, we define a notion of \emph{legal move}, similarly to communication protocols. For any two utterances $u_i,\ u_j \in \mU$, $u_i \neq u_j$, we say that:
\begin{itemize}
    \item $u_i$ is the \emph{target utterance} of $u_j$ iff the target of $u_j$ is the identifier of $u_i$, i.e., $u_i = (\_, \_, \CO_i, \ID)$ and $u_j = (\_, \ID, \CO_j, \_)$;

    \item $u_j$ is the \emph{legal move} after $u_i$ iff $u_i$ is the target utterance of $u_j$ and one of the following cases in Table~\ref{tab:legal-moves} holds.
    \end{itemize}

    \begin{table}\vspace{-6mm}
    \centering
        \caption{Locutions and responses}
   \label{tab:legal-moves}
    \begin{tabular}{|l|l|}
    \hline
    Locution $u_i$ &  Available responses $u_j$ \\
    \hline
    $\CO_i = \cla(\phi)$ & (1) $\CO_j = \off(\_ , \phi)$ if $\phi \in \cn(\{ \_ \})$, \\
                         & (2) $\CO_j =  \fa(\phi)$ if $\phi \in \mK$, \\
                         & (3) $\CO_j =  \cont(\_,\ \phi)$ where $\{\_, \phi \}$ is inconsistent; \\
    \hline
    $\CO_i = \fa(\phi)$ & $\CO_j = \cont(\_ , \phi)$ where $\{ \_, \phi \}$ is inconsistent; \\
    \hline
    $\CO_i = \off(\Delta, \phi)$ & (1) $\CO_j =  \cont(\_,\ \phi)$ where $\{\_, \phi \}$ is inconsistent, \\
     with $\phi \in \cn(\Delta)$ & (2) $\CO_j =  \cont(\_,\ \Delta)$ where $\{\_ \} \cup \Delta$ is inconsistent, \\
                                                         & (3) $\CO_j =  \off(\_, \beta_i)$ with $\beta_i \in \Delta$ and $\beta_j \in \cn(\{\_ \})$ \\
    \hline
    $\CO_i = \cont(\beta, \_)$ & (1) $\CO_j =  \cont(\_, \beta)$ where $\{ \_, \beta \}$ is inconsistent \\
                              & (2) $\CO_j =  \off(\_, \beta)$ with $\beta \in \cn(\{\_\})$. \\
    \hline
    \end{tabular}
\end{table}
     
An utterance is a legal move after another if any of the following cases happens: (1) it with content $\off$ contributes to expanding an argument; (2) it with content $\fa$ identifies a fact in support of an argument; (3) it with content $\cont$ starts the construction of a counter-argument. An utterance can be from the same agent or not.

\subsection{Dialogue Trees, Dialogues and Focused Sub-dialogues}
\label{sec:con-DT}

In essence, a dialogue is a sequence of utterances $u_1, \ldots, u_n$, each of which transforms the dialogue from one state to another.
To keep track of information disclosed in dialogues for P-SAFs, we define \emph{dialogue trees} constructed as the dialogue progresses.  These are subsequently used to determine \emph{successful dialogues} w.r.t argumentation semantics. 

A dialogue tree represents a dispute progress between a proponent and an opponent who take turns exchanging arguments in the form of formulas of a KB.
The proponent starts the dispute with their arguments and must defend against all of the opponent's attacks to win.
Informally, in a dialogue tree, the formula of each node represents an argument's conclusion or elements of the argument's support. 
A node is annotated \emph{unmarked} if its formula is only mentioned in the claim, but without any further examination, \emph{marked-non-fact} if its formula is the logical consequence of previous uttered formulas, and \emph{marked-fact} if its formula has been explicitly uttered as a fact in $\mK$.
A node is labelled $\po$ $(\op)$ if it is (directly or indirectly) for (against, respectively) the claim of the dialogue. The $\ID$ is used to identify the node’s corresponding utterance in the dialogue.
The nodes are connected in two cases: (1) they belong to the same argument, and (2) they form collective attacks between arguments. 
 We formally define dialogue trees and dialogues.

\begin{definition}
\label{def:dia-tree-DLAF}
Given a sequence of utterances $\delta = u_1, \ldots, u_n$, the \textbf{dialogue tree} $\mT (\delta)$ drawn from $\delta$ is a tree whose \emph{nodes} are tuples $(\tS,\ [\tT,\ \tL,\ \ID])$, where:
    \begin{itemize}
        \item $\tS$ is a formula in $\mL$,
        \item $\tT$ is either $\um$ (unmarked), $\nf$ (marked-non-fact), $\f$ (marked-fact),
        \item $\tL$ is either $\po$ or $\op$,
        \item $\ID$ is the identifier of the utterance $u_i$;
    \end{itemize}

and $\mT(\delta)$ is $\mT^{n}$ in the sequence $\mT^{1}, \ldots, \mT^{n}$ constructed inductively from $\delta$, as follows:
 \begin{enumerate}
     \item $\mT^{1}$ contains a single node: $(\phi, [\um ,\ \po,\ \id_1])$ where $\id_1$ is the identifier of the utterance $u_1 = (\_, \_, \cla(\phi), \id_1)$;

     \item  Let $u_{i+1} = (\_,\ \tg,\ \CO ,\ \id)$ be the utterance in $\delta$; $\mT^i$ be the $i$-th tree with the utterance $(\_,\ \_,\ \CO_{\tg},\ \tg)$ as the target utterance of $u_{i+1}$.
     Then $\mT^{i+1}$ is obtained from $\mT^i$ by $u_{i+1}$, if one of the following conditions holds: $(\tL, \tL_{\tg} \in \{\po, \op\}, \tL \neq \tL_{\tg})$:
    
     \begin{enumerate} [a)]
          \item If $\CO = \off(\Delta,\ \alpha)$ with $\Delta = \{\beta_1, \ldots, \beta_m \}$ and $\alpha \in \cn(\Delta)$,  then $\mT^{i+1}$ is obtained: 
       
        \begin{itemize}
            \item For all $\beta_j \in \Delta$, new nodes $(\beta_j, [\tT,\ \tL,\ \id])$ are added to the node $(\alpha, [\_,\ \tL ,\ \tg])$ of $\mT^i$. Here $\tT = \f$ if $\beta_j \in \mK$, otherwise $\tT = \nf$;

            \item  The node $(\alpha, [\_,\ \tL ,\ \tg])$ is replaced by $(\alpha, [\nf,\ \tL ,\ \tg])$;
        \end{itemize}

         \item If $\CO = \fa(\alpha)$ then $\mT^{i+1}$  is $\mT^i$ with the node $(\alpha,\ [\_,\ \tL,\ \tg])$ replaced by $(\alpha,\ [\f,\ \tL,\ \id])$;

         \item        
         If $\CO = \cont(\Delta, \eta)$ where $\Delta = \{\beta_1, \ldots, \beta_m \}$ and $\Delta \cup \{\eta \}$ is inconsistent, then $\mT^{i+1}$ is obtained by adding
         new nodes $(\beta_j, [\tT ,\ \tL ,\ \id])$, $(\tT = \f$ if $\beta_j \in \mK$, otherwise $\tT = \nf )$, as children of the node $(\eta, [\tT_{\tg} ,\ \tL_{\tg},\ \tg])$ of $\mT^i$, where $\tT_{\tg} \in \{ \f,\ \nf \}$.
         
     \end{enumerate}
 \end{enumerate}

 For such dialogue tree $\mT(\delta)$, the nodes labelled by $\po$ (resp., $\op$) are called the \emph{proponent nodes} (resp., \emph{opponent nodes}).
 We call the sequence $u_1, \ldots, u_n$ a \textbf{dialogue} $D(\phi)$ for $\phi$ where $\phi$ is the formula of the root in $\mT(\delta)$.
 \end{definition}

 %

 This dialogue tree can be seen as a concrete representation of an \emph{abstract dialogue tree} defined in~\cite{loanho_2024}. 
 Here, the nodes represent formulas and the edges display either the monotonic inference steps used to construct arguments or the attack relations between arguments. A group of nodes in a dialogue tree with the same label $\po$ (or $\op$) corresponds to the proponent (or opponent) argument in the abstract dialogue tree.


\begin{definition} [Focused sub-dialogues]
\label{def:focused-sub-dia}
$\delta^{\prime}$ is called a \emph{focused sub-dialogue} of a dialogue $\delta$  iff it is a dialogue for $\phi$ and, for all utterances $u \in \delta^{\prime}$, $u \in \delta$. We say that $\delta$ is the \emph{full-dialogue} of $\delta^{\prime}$ and $\mT(\delta^{\prime})$ drawn from $\delta^{\prime}$  is the sub-tree of $\mT(\delta)$.
 
\end{definition}

If there are no utterances for both proponents and opponents in a dialogue tree from a dialogue $\delta$, then $\delta$ is called \emph {terminated}.
Note that a dialogue can be "incomplete", which means that it ends before the utterances related to determining success are claimed. To prevent this from happening we assume that dialogues are \emph{complete}, i.e. that there are no "unsaid" utterances (with the content $\fa$, $\off$ or $\cont$) in such dialogue that would bring important arguments to determine success. This assumption will ease the proof of soundness result later. 

 \begin{example} [Continue Example~\ref{ex:KB-arg}]
\label{ex:tab-dia}
When users received the answer "$(\vi)$ \emph{is possible researcher}", they would like to know "\emph{Why is this the case?}". The system will explain to the users through the natural language dispute agreement that the agent $a_1$ is persuading $a_2$ to agree that $\vi$ is a researcher. This dispute agreement is formally modelled by an explanatory dialogue $D(\rese(\vi)) = \delta$ as in Figure~\ref{tab:dia}.

\begin{figure} \vspace{-8mm}
\centering
    \includegraphics [scale = 0.85]{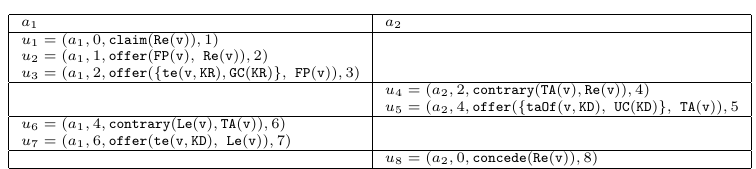}\vspace{-3mm}
    \caption{\scriptsize Given $\mL_t$ is $\mK_1$, a dialogue $D(\rese(\vi))$ $= u_1, \ldots, u_9 $~for $q_1 = \rese(\vi)$}
        \label{tab:dia}
\end{figure}

Figure~\ref{fig:construct-tree} illustrates how to fully construct a dialogue tree $\mT(\delta)$  from $D(\rese(\vi)) = \delta$. 
To avoid confusing users, after the construction processing, we display the final dialogue tree $\mT(\delta)$ with necessary labels, such as formulas, $\po$ and $\op$, in  Figure~\ref{fig:tree-user}.
The line indicates that children conflict with their parents. The dotted line indicates that children are implied from their parents by inference rules.
From this tree, the system provides a dialogical explanation in natural language as shown in Example~\ref{ex:motivation-ex}.
\end{example}

\begin{figure}  \vspace{-8mm}
\centering   
\includegraphics[scale = 0.6]{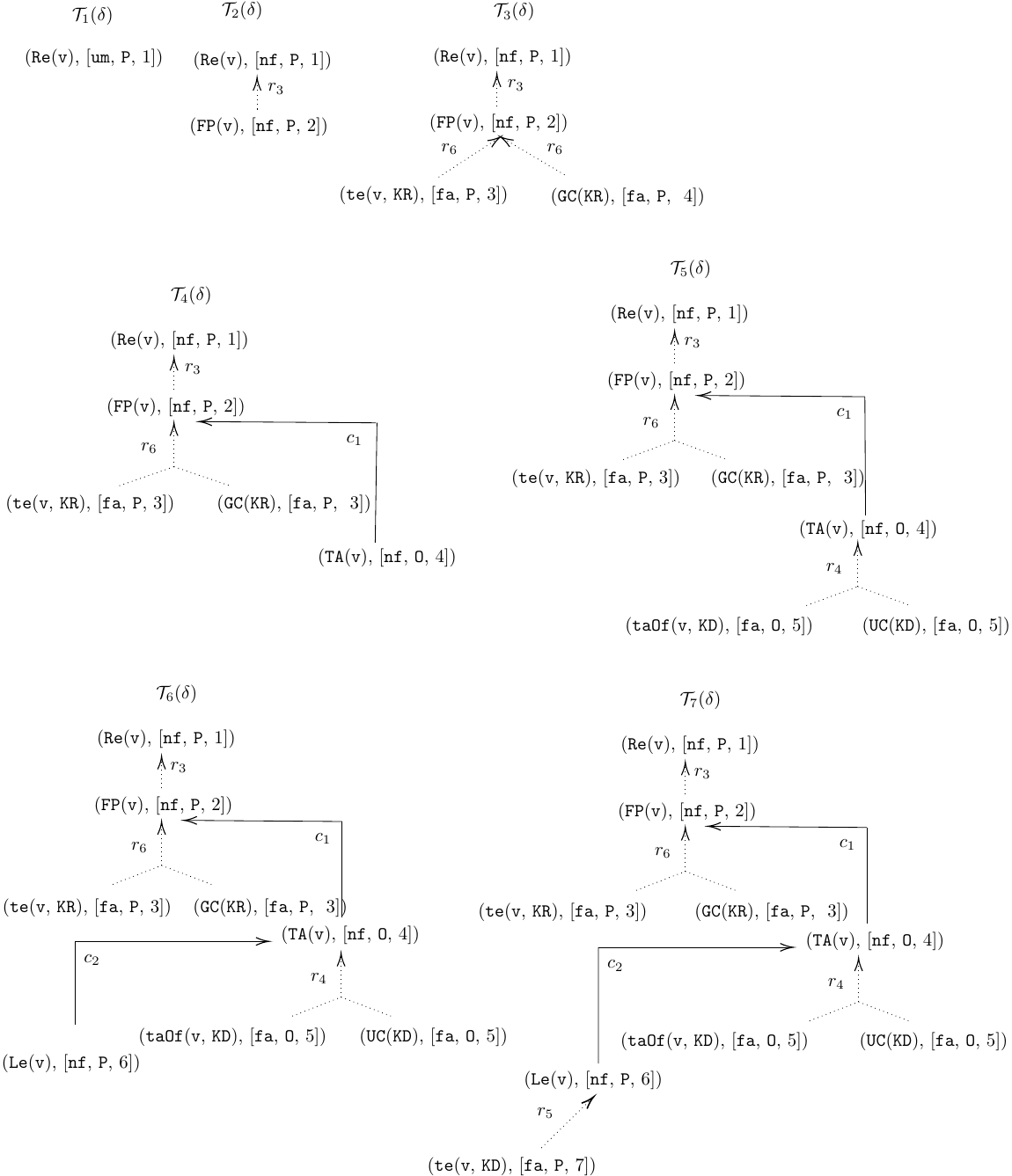}
\caption{Construction of the dialogue tree $\mT(\delta) = \mT_{7}(\delta)$ drawn from $D(\rese(\vi))$.}
\label{fig:construct-tree}
\end{figure}

\begin{figure}  \vspace{-8mm}
\centering   
\includegraphics[scale = 0.55]{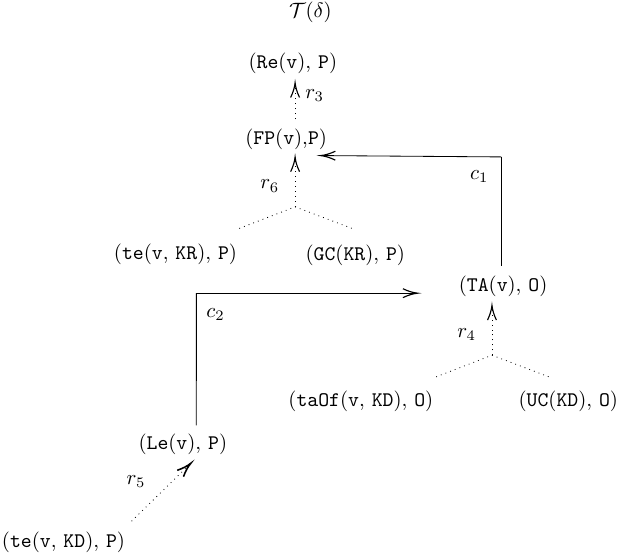}
\caption{A final version of the dialogue tree $\mT(\delta)$ is displayed for users}
\label{fig:tree-user}
\end{figure}

\subsection{Focused Dialogue Trees}

To determine and explain the arguments of acceptability (wrt argumentation semantics) by using dialogues/ dialogue trees, we present a notion of \emph{focused dialogue trees} that will be needed for the following sections. 
This concept is useful because it allows us to show a \emph{correspondence principle} between dialogue trees and \emph{abstract dialogue trees} defined in~\cite{loanho_2024}~\footnote{
We reproduce the notion of abstract dialogue trees and introduce the correspondence principle in Appendix~\ref{app:pre}.
Here we briefly describe the concept of abstract dialogue trees: an abstract dialogue tree is a tree where nodes are labeled with arguments, and edges represent attacks between arguments. }.
By the correspondence principle, we can utilize the results from~\cite{loanho_2024} to obtain the important results in Section~\ref{sec:soundness} and~\ref{sec:completeness}.

Observe that a dialogue $\delta$ can be seen as a collection of several (independent) focused sub-dialogues $\delta_1, \ldots, \delta_n$. The dialogue tree $\mT(\delta_i)$ drawn from the focused sub-dialogue $\delta_i$ is a subtree of $\mT(\delta)$ and corresponds to the abstract dialogue tree (defined in~\cite{loanho_2024}) (for an argument for $\phi$). Each such subtree of $\mT(\delta)$ has the following properties: (1) $\phi$ is supported by a single proponent argument; (2) An opponent argument is attacked by either a single proponent argument or a set of collective proponent arguments; (3) A proponent argument can be attacked by either multiple single opponent arguments or sets of collective opponent arguments. We call a tree with these properties the \emph{focused dialogue tree}.

\begin{definition} [Focused dialogue trees]
\label{def:t-focused}
A dialogue tree $\mT(\delta)$ is \emph{focused} iff
\begin{enumerate}
    \item all the immediate children of the root node have the same identifier (that is, are part of a single utterance);
    
    \item all the children labelled~$\po$ of each potential argument labelled~$\op$
        have the same identifier (that is, are part of a single utterance)
\end{enumerate}
\end{definition}

In the above definition, we call child of a potential argument a node that
is child of any of the nodes of the potential argument.

\begin{remark}
Focused dialogue trees and their relation to abstract dialogue trees are crucial for proving the important results in Section~\ref{sec:soundness} and~\ref{sec:completeness}. We refer to Appendix~\ref{app:proof-soundness} for details.    
\end{remark}

\begin{example} Consider a query $q_3 = A(a) $ to a KB $\mK_3 = (\mR_3, \mC_3, \mF_3)$ where 
\begin{align*}
    \mR_3 = & \{r_1: C(x) \land B(x) \rightarrow A(x),\ r_2: D(x) \rightarrow A(x) \} \\
    \mC_3 = & \{ D(x) \land C(x) \rightarrow \bot ,\ E(x) \land C(x) \rightarrow \bot \} \\
    \mF_3 = & \{B(a) , C(a), D(a), E(a) \}
\end{align*}
  Figure~\ref{fig:non-foc-tree} (Left) shows a non-focused dialogue tree drawn for a dialogue $D(A(a)) = \delta$.  Figure~\ref{fig:non-foc-tree}(Right) shows a focused dialogue tree $\mT(\delta_1)$ drawn for a sub-dialogue $\delta_1$ of $\delta$. This tree is the sub-tree of $\mT(\delta)$.
\end{example}

\begin{figure}
\centering
\begin{tikzpicture}
    \node (dt) at (0,0) {\includegraphics[scale=0.5]{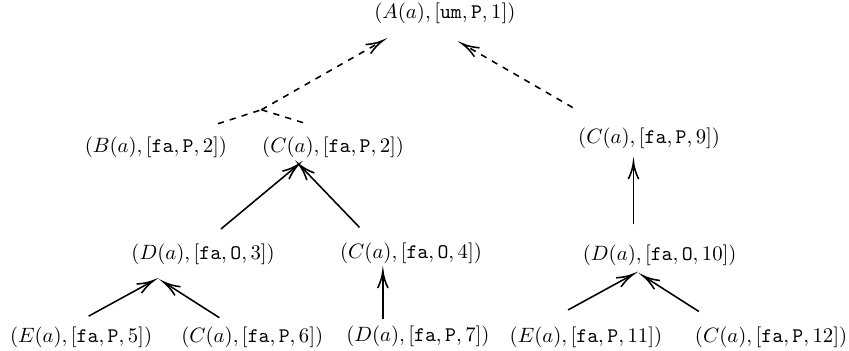}};
    \node (d1) at (6, 0) {\includegraphics[scale=0.55]{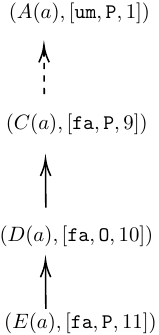}};
\end{tikzpicture}
\caption{
Left: A non-focused dialogue tree.
Right: A focused dialogue tree $\mT(\delta_1)$.
}
\label{fig:non-foc-tree}
\end{figure}

\section{Results of the Paper}
In this section, we study how to use a novel explanatory dialogue model to determine and explain the acceptance of a formula $\phi$ wrt argumentation semantics.

Intuitively, a \textit{successful dialogue} for formula $\phi$ wrt argumentation semantics is a \textit{dialectical proof procedure} for $\phi$. To argue for the usefulness of the dialogue model, we will study \emph{winning conditions} ("conditions" for short) for a successful dialogue to be \textit{sound} and \textit{complete} wrt argumentation semantics. To do so, we use dialogue trees. When the agent decides what to utter or whether a terminated dialogue is \emph{successful}, it needs to consider the current dialogue tree and ensure that its new utterances will keep the tree fulfilling desired \textit{properties}. 
Thus, the dialogue tree drawn from a dialogue can be seen as \emph{commitment store}~\cite{prakken_2006} holding information disclosed and used in the dialogue. Successful dialogues, in this sense, can be regarded as explanations for the acceptance of a formula.

Before continuing, we present preliminary notions/results to prove the soundness and completeness results.

\subsection{Notions for Soundness and Completeness Results}
Let us introduce notions that will be useful in the next sections. 
These notions include: \textbf{potential argument} obtained from a dialogue tree, \textbf{collective attacks} against a potential argument in a dialogue tree, and \textbf{P-SAF} drawn from a dialogue tree. 

A \emph{potential argument} is an argument obtained from a dialogue tree.


\begin{definition} \label{def:arg-t} A \emph{potential argument} obtained from a dialogue tree $\mT(\delta)$ is a \emph{sub-tree} $\mT^{s}$ of  $\mT(\delta)$ such that:
\begin{itemize}
    \item all nodes in $\mT^{s}$ have the same label (either $\po$ or $\op$);
    \item if there is an utterance $(\_ , \_ , \off(\Delta, \alpha), \id)$
        and a node $(\beta_i, [\_ , \tL , \id])$ in $\mT^{s}$ with $\beta_i \in \Delta$,
        then all the nodes
    $(\beta_1, [\_ , \tL , \id]), \ldots, (\beta_m, [\_ , \tL , \id])$ are in $\mT^{s}$
    \item for every node $(\alpha, [\nf, \tL, \_])$ in $\mT^{s}$,
        all its immediate children in $\mT^{s}$ have the same identifier (they belong to a single utterance).
\end{itemize}  
The formula $\phi$ in the root of $\mT^{s}$ is the \emph{conclusion}. The set of the formulas $H$ held by the descended nodes in $\mT^{s}$, i.e., $H = \{\beta \mid (\beta, [\f,\ \_,\ \_]) \text{ is a node in } \mT^{s}\}$, is the \emph{support} of $\mT^{s}$. 
A potential argument obtained from a dialogue tree is a \emph{proponent (opponent) argument} if its nodes are labelled $\po$ ($\op$, respectively). 
\end{definition}

To shorten notation, we use the term "an argument for $\phi$" instead of the term "an argument with the conclusion $\phi$".

\begin{example} [Continue Example~\ref{ex:tab-dia}]
Figure~\ref{fig:comple-tree} shows two potential arguments obtained from $\mT(\delta)$.
\end{example}

Potential arguments correspond to the conventional P-SAF arguments.

\begin{lemma}
\label{lem:potential-arg}
A potential argument $\mT^{s}$ corresponds to an argument for $\phi$ supported by $H$ as in conventional P-SAF (in Definition~\ref{def:ab-arg}).
\end{lemma}

\begin{proof} This lemma is trivially true as a node in a potential argument can be mapped to a node in a conventional P-SAF argument (in Definition~\ref{def:ab-arg}) by dropping the tag $\tT$ and the identifier $\ID$.
\end{proof}

We introduce \emph{collective attacks} against a potential argument, or a sub-tree, in a dialogue tree. This states that a potential argument is \emph{attacked} when there exist nodes within the tree that are children of the argument. Formally:

\begin{definition}
Let $\mT(\delta)$ be a dialogue tree and $\mT^{s}$ be a potential argument obtained from $\mT(\delta)$. $\mT^{s}$ is \emph{attacked} iff there is a node $N = (\tL, [\tT, \_, \_])$ in $\mT^{s}$, with $\tL \in \{\po, \op\}$ and $\tT \in \{\f, \nf\}$, such that $N$ has children $M_1, \ldots, M_k$ labelled by $\tL^{\prime} \in \{\po, \op\}\setminus\{\tL\}$ in $\mT(\delta)$ and the children have the same identifier.
    
    We say that the sub-trees rooted at $M_j$ ($1 \leq j \leq k$) \emph{attacks} $\mT^{s}$.

\end{definition}

\begin{definition} (A) \emph{P-SAF drawn from} $\mT(\delta)$ is $ \mAF_ \delta = (\Arg_{\delta}, \Att_{\delta})$, where 
\begin{itemize}
    \item $\Arg_{\delta}$ is the set of potential arguments obtained from $\mT(\delta)$;
    \item $\Att_{\delta}$ contains the attacks between the potential arguments.
\end{itemize}
\end{definition}
Since $\mT(\delta)$ is drawn from $\delta$, we can say $\mAF_ \delta$ drawn from $\delta$ instead.

 As in~\cite{DUNG2006114}, two useful concepts that are used for our soundness result in the next sections are the \emph{defence set} and the \emph{culprits} of a dialogue tree. 
 \begin{definition}
 Given a dialogue tree $\mT(\delta)$, 
 \begin{itemize}
 \item The \emph{defence set} $\mDE(\mT(\delta))$ is the set of  all facts $\alpha$ in proponent nodes of the form $N = (\alpha,[ \f, \po, \_])$ such that $N$ is in a potential argument;

\item The \emph{culprits} $\mCU(\mT(\delta))$ is the set of facts $\beta$ in opponent nodes $N = (\beta, [\f, \op, \_])$ such that $N$ has the child node $N^{\prime} = (\_,[ \_, \po, \_])$ and $N$ and $N^{\prime}$ are in potential arguments.
\end{itemize}
\end{definition}

\begin{example}
Figure~\ref{fig:comple-tree} (Left) gives the focused dialogue tree drawn from the dialogue $D(\rese(\vi))$ in Example~\ref{ex:tab-dia}. The defence set is  $\{\te(\vi, \kr), \gc(\kr), \te(\vi, \kd)\}$; the culprits are $\{\teAs(\vi, \kd), \uc(\kd)\}$.
\end{example}

\begin{figure}
\centering
\begin{tikzpicture}
    \node (dt) at (0,0) {\includegraphics[scale=0.6]{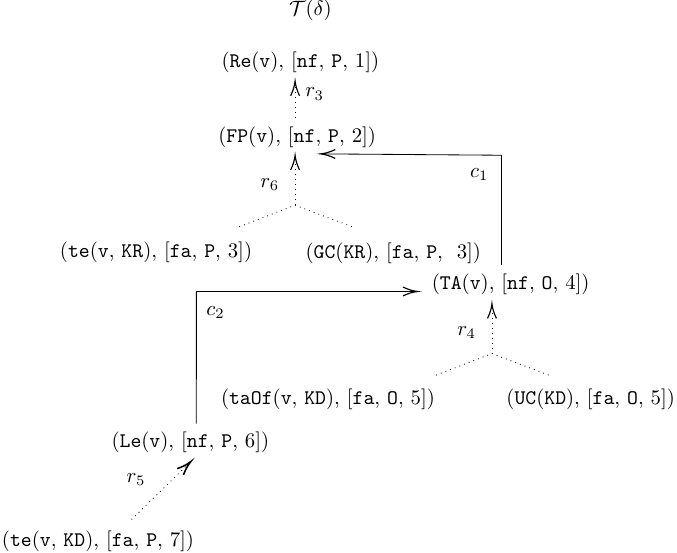}};
    \node (d1) at (7,1.5) {\includegraphics[scale=0.65]{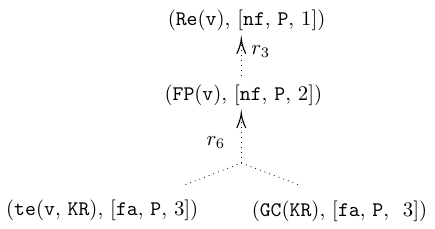}};
    \node (d2) at (7,-2) {\includegraphics[scale=0.65]{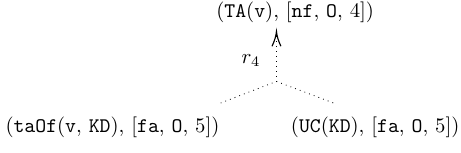}};
\end{tikzpicture}
\caption{
Left:
A focused dialogue tree $\mT(\delta)$ drawn from $D(\rese(\vi))$ in Table~\ref{tab:dia}.
Right: Some potential argument obtained from $\mT(\delta)$.
}
\label{fig:comple-tree}
\end{figure}

\subsection{Soundness Results}
\label{sec:soundness}

\subsubsection{Computing credulous acceptance}
\label{sec:credulously-success}

We present winning conditions for a \textit{credulously successful dialogue} to prove whether a formula is credulously accepted under admissible/ preferred/ stable semantics. 

Let us sketch the idea of a dialectical proof procedure for computing the credulous acceptance as follows:
Assume that a (dispute) dialogue between an agent $a_1$ and $a_2$ in which $a_1$ persuades $a_2$ about its belief "$\phi$ is accepted". Two agents take alternating turns in exchanging their arguments in the form of formulas. When the (dispute) dialogue progresses, we are increasingly building, starting from the root $\phi$, a dialogue tree. Each node of such tree, labelled with either $\po$ or $\op$, corresponds to an utterance played by the agent. The credulous acceptance of $\phi$ is proven if $\po$ can win the game by ending the dialogue in its favour according to a “\textit{last-word}” principle. 

To facilitate our idea, we introduce the properties of a dialogue tree:\textit{ patient, last-word, defensive and non-redundant.}


Firstly, we restrict dialogue trees to be \emph{patient}. This means that agents wait until a potential argument has been fully constructed before beginning to attack it. Formally: A dialogue tree $\mT(\delta)$ is \emph{patient} iff for all nodes $N = (\_, [\f,\_,\_])$ in $\mT(\delta)$, $N$ is in (the support of) a potential argument obtained from $\mT(\delta)$.
Through this paper, the term "dialogue trees" refers to \emph{patient dialogue trees}.

We now present the "last-word" principle to specify a winning condition for the proponent. In a dialogue tree, $\po$ wins if either $\po$ finishes the dialogue tree with the un-attacked facts (Item 1), or any attacks used by $\op$ have been attacked with valid counter attacks (Item 2). Formally:

\begin{definition} A focused dialogue tree $\mT(\delta)$ is \emph{last-word} iff

    \begin{enumerate}
     \item for all leaf nodes $N$ in $\mT(\delta)$, $N$  is the form of $(\_, [\f, \po, \_])$, and
     

     \item if a node $N$ is of the form $(\_, [\tT, \op, \_])$ with $\tT \in \{\f, \nf\}$, then $N$ is in a potential argument and $N$ is properly attacked.   
    \end{enumerate}
\end{definition}
In the above definition, we say that a node $N$ of a potential argument is attacked, meaning that $N$ has children labelled by $\po$ with the same identifier.

 The definition of "last-word" incorporates the requirement that a set of potential arguments $\mS$ (supported by the defence set) attacks every attack against $\mS$. However, it does not include the requirement that $\mS$ does not attack itself. This requirement is incorporated in the definition of \emph{defensive dialogue trees}. 

\begin{definition} 
\label{def:defensive-tree}
A focused dialogue tree $\mT(\delta)$ is \emph{defensive} iff it is
\begin{itemize}
    \item last-word, and
    \item no formulas $\Delta$ in opponent nodes belong to $\mDE(\mT(\delta))$ such that $\Delta \cup \mDE(\mT(\delta))$ is inconsistent.
\end{itemize}
\end{definition}


In admissible dialogue trees, nodes labelled $\po$ and $\op$ within potential arguments can have common facts when considering potential arguments that attack or defend others.
However, potential arguments with nodes sharing common facts cannot attack proponent potential arguments whose facts are in the defence set.
Let us show this in the following example.

\begin{example}
\label{ex:a-succ}
Consider a query $q_4 = A(a) $ to a KB $\mK_4 = (\mR_4, \mC_4, \mF_4)$ where 
\begin{align*}
    \mR_4 = &\emptyset \\
    \mC_4 = & \{c_1 : A(x) \land \ B(x) \land C(x) \rightarrow \bot \} \\
    \mF_4 = & \{A(a), B(a) , C(a) \}
\end{align*}
Consider the focused dialogue tree $\mT(\delta_i)$ (see Figure~\ref{fig:tree-ex} (Left)) drawn from the focused sub-dialogue $\delta_i$ of a dialogue $D(A(a)) = \delta$. The defence set $\mDE(\mT(\delta_i)) = \{A(a), C(a)\}$; the culprits $\mCU(\mT(\delta_i)) = \{B(a), C(a)\}$.
We have $\mDE(\mT(\delta_i)) \cap \mCU(\mT(\delta_i)) = \{C(a)\}$. It can seen that $\{C(a)\} \cup \mDE(\mT(\delta_i))$ is inconsistent. In other words, there exists a potential argument, say $A$, such that $\{C(a)\}$ is the support of $A$, and $A$ cannot attack any proponent argument supported by $\mDE(\mT(\delta_i))$. Clearly, $\mDE(\mT(\delta_i))$ and $\mCU(\mT(\delta_i))$ have the common formula, but the set of arguments supported by $\mDE(\mT(\delta_i))$ does not attack itself.
\end{example}

\begin{figure}
\centering
\begin{tikzpicture}
    \node (dt) at (0,0) {\includegraphics[scale=0.6]{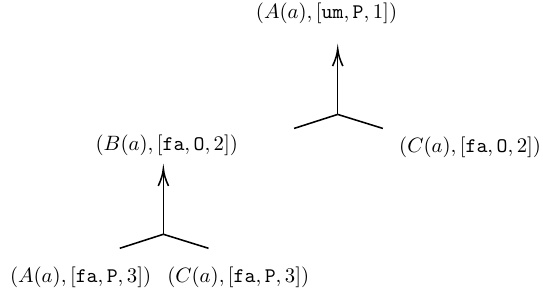}};
    \node (d1) at (5, 0) {\includegraphics[scale=0.55]{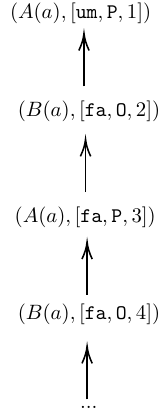}};
\end{tikzpicture}
\caption{
Left: A focused dialogue tree $\mT(\delta_i)$.
Right: An infinite dialogue tree.
}
\label{fig:tree-ex}
\end{figure}


From the above observation, it follows immediately that.

\begin{lemma}
    Let $\mT(\delta)$ be a defensive dialogue tree. The set of proponent arguments (supported by $\mDE(\mT(\delta))$) does not attack itself in the P-SAF drawn from $\delta$.
\end{lemma}


Consider the following dialogue to see why the "non-redundant" property is necessary.

\begin{example}
\label{ex:infinite-credulous}
Consider a query $q_5 = A(a)$ to a KB $\mK_5 = (\mR_5, \mC_5, \mF_5)$ where
\begin{align*}
    \mR_5 = &\emptyset \\
    \mC_5 = & \{ A(x) \land \ B(x) \rightarrow \bot \} \\
    \mF_5 = & \{A(a), B(a) \}
\end{align*}
Initially, an argument $A_1$ asserts that "$A(a)$ is accepted" where $A(a)$ is at the $\po$ node. $A_1$ is attacked by $A_2$ by using $B(a)$ that is at the $\op$ node. $A_1$ counter-attacks $A_2$ by using $A(a)$, then $A_2$ again attacks $A_1$ by using $B(a)$, ad infinitum (see Figure~\ref{fig:tree-ex} (Right)). Hence $\po$ cannot win.
%
%
Since the grounded extension is empty, $A(a)$ is not groundedly accepted in the P-SAF, thus $\po$ should not win under the grounded semantics. Since $A(a)$ is credulously accepted in the P-SAF, we expect that $\po$ can win in a terminated dialogue under the credulous semantics.
\end{example}


To ensure credulous acceptance, all possible opponent nodes must be accounted for. But if such a parent node is already in the dialogue tree, then deploying it will not help the opponent win the dialogues. To avoid this, we define a dialogue tree to be \emph{non-redundant}.

\begin{definition}
\label{def:non-re}
     A focused dialogue tree $\mT(\delta)$ is \emph{non-redundant} iff for any two nodes $N_1 = (\beta, [\f, \tL, \id_1])$  and $N_2 = (\beta, [\f, \tL, \id_2])$ with $\tL \in \{\po, \op\}$ and $N_1 \neq N_2$, if $N_1$ is in a potential argument $\mT_1^{s}$ and $N_2$ is in a potential argument $\mT_2^{s}$, then $\mT_1^{s} \neq \mT_2^{s}$.
\end{definition}

In Definition~\ref{def:non-re}, when comparing two arguments, we compare their respective proof trees. Here, we only consider the formula and the tag of each node in the tree, disregarding the label and identifier of the node.

The following theorem establishes credulous soundness for admissible semantics.

\begin{restatable}{theorem} {thmcredulous} \label{thm:adm}
 Let $\delta$ be a dialogue for a formula $\phi \in \mL$. If there is a dialogue tree $\mT(\delta_i)$ drawn from a focused sub-dialogue $\delta_i$ of $\delta$ such that it is defensive and non-redundant, then 
  \begin{itemize}
      \item $\delta$ is admissible-successful; 
      \item $\phi$ is credulously accepted under $\adm$ in $\mAF_ \delta$ drawn from $\delta$ (supported by $\mDE(\mT(\delta_i))$.
\end{itemize}
\end{restatable}

The proof of this theorem is in Appendix~\ref{app:proof-soundness}.

We can define a notion of \emph{preferred-successful dialogue} with a formula accepted under $\prf$ in the P-SAF framework drawn from the dialogue. Since every admissible set (of arguments) is necessarily contained in a preferred set (see~\cite{Dung95,Nielsen2007}), and every preferred set is admissible by definition, trivially a dialogue is preferred-successful iff it is admissible-successful. The following theorem is analogous to Theorem~\ref{thm:adm} for $\prf$ semantics.

\begin{restatable} {theorem} {thmpreferred} 
\label{thm:prf-stb}
Let $\delta$ be a dialogue for a formula $\phi \in \mL$. If there is a dialogue tree $\mT(\delta_i)$ drawn from a focused sub-dialogue $\delta_i$ of $\delta$ such that it is defensive and non-redundant, then $\delta$ is preferred-successful and $\phi$ is credulously accepted under $\prf$ in $\mAF_ \delta$ drawn from $\delta$ (supported by $\mDE(\mT(\delta_i))$.    
\end{restatable}

\begin{proof} [Sketch]
The proof of this theory follows the fact that every preferred dialogue tree is an admissible dialogue tree. Thus, the proof of this theorem is analogous to those of Theorem~\ref{thm:adm}.
\end{proof}

\begin{remark}
    We can similarly define a notion of \emph{stable dialogue trees} for a formula accepted under $\stb$ in the P-SAF. Since stable and preferred sets coincide, trivially a dialogue tree is stable iff it is defensive and non-redundant. Thus we can use the result of Theorem~\ref{thm:prf-stb} for stable semantics.
\end{remark}

\subsubsection{Computing grounded acceptance}
We present winning conditions for a \textit{groundedly successful dialogue} to determine grounded acceptance of a given formula. 
The conditions require that whenever $\op$ could advance any evidence, $\po$ still wins.
This requirement is incorporated in dialogue trees being defensive.
Note that credulously successful dialogues for computing credulous acceptance also require dialogue trees to be defensive (see in Theorem~\ref{thm:adm}).
However, the credulously successful dialogues cannot be used for computing the grounded acceptance, as shown by Example~\ref{ex:infinite-credulous}.
In Example~\ref{ex:infinite-credulous}, it would be incorrect to infer from the depicted credulously successful dialogue that $A(a)$ is groundedly accepted as the grounded extension is empty. Note that the dialogue tree for $A(a)$ is infinite.
From this observation, it follows that the credulously successful dialogues are not sound for computing grounded acceptance. Since all dialogue trees of a formula that is credulously accepted but not groundedly accepted can be infinite,
we could detect this situation by checking if constructed dialogue trees are infinite. This motivates us to consider "\textit{finite}" dialogue trees as a winning condition.

The following theorem establishes the soundness of grounded acceptance.
\begin{restatable} {theorem} {thmground}   
\label{thm:ground}
Let $\delta$ be a dialogue for a formula $\phi \in \mL$. If there is a dialogue tree $\mT(\delta_i)$ drawn from a focused sub-dialogue $\delta_i$ of $\delta$ such that it is defensive and finite, then
  \begin{itemize}
    \item $\delta$ is groundedly-successful;
      \item $\phi$ is groundedly accepted under grounded semantics in $\mAF_ \delta$ drawn from $\delta$ (supported by $\mDE(\mT(\delta_i))$.
\end{itemize}
\end{restatable}

The proof of this theorem is in Appendix~\ref{app:proof-soundness}.

\subsubsection{Computing sceptical acceptance}

Inspired by~\cite{DUNG2007642}, to determine the sceptically acceptance of an argument for $\phi$, we verify the following:
(1) There exists an admissible set of arguments $S$ that includes the argument for $\phi$;
(2) For each argument $A$ attacking $S$, there exists no admissible set of arguments containing $A$.
These steps can be interpreted through the following winning conditions for a \emph{sceptical successful dialogue} to compute the sceptical acceptance of $\phi$:
\begin{enumerate}
    \item $\po$ wins the game by ending the dialogue,
    \item none of $\op$ wins by the same line of reasoning.
\end{enumerate}
This perspective allows us to introduce a notion of \emph{ideal dialogue trees}.

\begin{definition}
\label{def:tree-ideal}
     A defensive and non-redundant dialogue tree $\mT(\delta)$ is \emph{ideal} iff none of the opponent arguments obtained from $\mT(\delta)$ belongs to an admissible set of potential arguments in $ \mAF_ \delta$ drawn from $\mT(\delta)$.
\end{definition}

The following result sanctions the soundness of sceptical acceptance.

\begin{restatable} {theorem}{thmsceptical}
\label{thm:scep}
Let $\delta$ be a dialogue for a formula $\phi \in \mL$. If there is a dialogue tree $\mT(\delta)$ drawn from $\delta$ such that it is ideal, then
\begin{itemize}
    \item $\delta$ is sceptically-successful;
    \item $\phi$ is sceptically accepted under $\sem$ in $\mAF_ \delta$ drawn from $\delta$ (supported by $\mDE(\mT(\delta))$, where $\sem \in \{\adm, \prf, \stb\}$.
\end{itemize} 
\end{restatable}
The proof of this theorem is in Appendix~\ref{app:proof-soundness}.

\subsection{Completeness Results}
\label{sec:completeness}
We now present completeness. 
In this work, dialogues viewed as dialectical proof procedures are sound but not always complete in general.
The reason is that the dialectical proof procedures might enter a non-terminating loop during the process of argument constructions, which leads to the incompleteness wrt the admissibility semantics.
To illustrate this, we refer to Example 1 using logic programming in~\cite{ThangDP22} for an explanation.
We also provide another example using \datalogPM.

\begin{example} Consider a query $q_6 = P(a)$ to a \datalogPM KB $\mK_6 = (\mR_6 , \mC_6 , \mF_6)$ where
\begin{align*}
    \mR_6 = & \{r_1: P(x) \rightarrow Q(x), r_2: Q(x) \rightarrow P(x) \} \\
    \mC_6 = & \{ P(x) \land R(x) \rightarrow \bot \} \\
    \mF_6 = & \{P(a) , R(a)\}
\end{align*}
The semantics of the corresponding P-SAF $\mAF_4$ are determined by the arguments illustrated in Figure~\ref{fig:infinite-loop}. The result should state that "$P(a)$ is a possible answer" as the argument $B_1$ for $P(a)$ is credulously accepted under the admissible sets $\{B_1\}$  and $\{B_2\}$ of $\mAF_4$. 
 But the dialectical proof procedures fail to deliver the admissible set $\{ B_1 \}$ wrt $\mAF_4$
as they could not overcome the non-termination of the process to construct an argument $B_1$ for $P(a)$ due to the “infinite loop”. 
    
\end{example}
\begin{figure}
    \centering
    \includegraphics[width=0.25\linewidth]{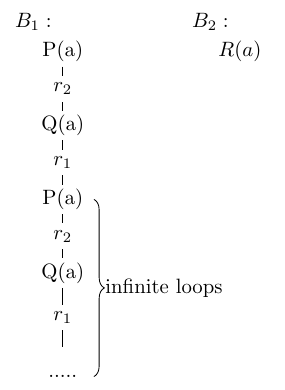}
    \caption{Arguments of $\mAF_4$}
    \label{fig:infinite-loop}
\end{figure}



%
%

Intuitively, since the dialogues as dialectical proof procedures (implicitly) incorporate the computation of arguments top-down, the process of argument construction should be finite (also known as finite tree-derivations in the sense of Definition~\ref{def:ab-arg}) to achieve the completeness results. Thus, we restrict the attention to decidable logic with cycle-restricted conditions that its corresponding P-SAF framework produces arguments to be computed finitely in a top-down fashion. For example, given a \datalogPM KB $\mK = (\mR, \mC, \mF)$, the \emph{dependency graph} of the KB  as defined in~\cite{HechamBC17}  consists of the vertices representing the atoms and the edges from an atom $u$ to $v$ iff $v$ is obtained from $u$ (possibly with other atoms) by the application of a rule in $\mR$. The intuition behind the use of the dependency graph is that no infinite tree-derivation exists if the dependency graph of KB is acyclic. By restricting such acyclic dependency graph condition, the process of argument construction in the corresponding  P-SAF of the KB $\mK$ will be finite, which leads to the completeness of the dialogues wrt argumentation semantics. The following theorems show the completeness of credulous acceptances wrt admissible semantics.

 




\begin{restatable} {theorem}{compadm}
\label{thm:com-adm}
Let $\delta$ be a dialogue for a formula $\phi \in \mL$. If $\phi$ is credulously accepted under $\adm$ in $\mAF_ \delta$ drawn from $\delta$ (supported by $\mDE(\mT(\delta))$)
and $\delta$ is admissible-successful, then there is a defensive and non-redundant dialogue tree $\mT(\delta_i)$ for $\phi$ drawn from a focused sub-dialogue $\delta_i$ of $\delta$.
\end{restatable}

The proof of this theorem is in Appendix~\ref{app:proof-completeness}.

The following theorem is analogous to Theorem~\ref{thm:com-adm} for preferred semantics.
\begin{restatable} {theorem}{comppreferred}
\label{thm:com-prf}
Let $\delta$ be a dialogue for a formula $\phi \in \mL$. If $\phi$ is credulously accepted under $\prf$ in $\mAF_ \delta$ drawn from $\delta$ (supported by $\mDE(\mT(\delta))$)
and $\delta$ is preferred-successful, then there is a defensive and non-redundant dialogue tree $\mT(\delta_i)$ for $\phi$ drawn from a focused sub-dialogue $\delta_i$ of $\delta$.
\end{restatable}

\begin{proof} [Sketch]
The proof of this theory follows the fact that every preferred-successful dialogue is an admissible-successful dialogue. Thus, the proof of this theorem is analogous to those of Theorem~\ref{thm:com-adm}.
\end{proof}
Theorem~\ref{thm:com-ground} presents the completeness of grounded acceptances.
\begin{restatable} {theorem}{compground}
\label{thm:com-ground}
Let $\delta$ be a dialogue for a formula $\phi \in \mL$. If $\phi$ is groundedly accepted under $\grd$ in $\mAF_ \delta$ drawn from $\delta$ (supported by $\mDE(\mT(\delta))$) and $\delta$ is groundedly-successful, then there is a defensive and finite dialogue tree $\mT(\delta_i)$ for $\phi$ drawn from a focused sub-dialogue $\delta_i$ of $\delta$.
\end{restatable}

The proof of this theorem is in Appendix~\ref{app:proof-completeness}.

Theorem~\ref{thm:com-scep} presents the completeness of sceptical acceptances.

\begin{restatable} {theorem}{compsceptical}
\label{thm:com-scep}
Let $\delta$ be a dialogue for a formula $\phi \in \mL$. If $\phi$ is sceptically accepted under $\sem$ in $\mAF_ \delta$ drawn from $\delta$ (supported by $\mDE(\mT(\delta))$), where $\sem \in \{\adm, \prf, \stb\}$, and $\delta$ is sceptically-successful, then there is an ideal dialogue tree $\mT(\delta)$ for $\phi$ drawn from $\delta$.
\end{restatable}

The proof of this theorem is in Appendix~\ref{app:proof-completeness}.

\subsection{Results for a Link between Inconsistency-Tolerant Reasoning and Dialogues}

In Section~\ref{sec:soundness} and~\ref{sec:completeness}, we demonstrated the use of dialogue trees to determine the acceptance of a formula in the P-SAF drawn from the dialogue tree. As a direct corollary of Theorem~\ref{thm:ab-link} -\ \ref{thm:com-scep}, we show how to determine and explain the entailment of a formula in KBs by using dialogue trees, which was the main goal of this paper.

\begin{corollary}
    Let $ \mK$ be a KB, $\phi$ be a formula in $\mL$. Then $\phi$ is entailed in
    \begin{itemize}
        \item some maximal consistent subset of $\mK$ iff there is a defensive and non-redundant dialogue tree $\mT(\delta)$ for $\phi$.
        
        \item the intersection of maximal consistent subsets of $\mK$ iff there is a defensive and finite dialogue tree $\mT(\delta)$ for $\phi$.
        
        \item all maximal consistent subsets of $\mK$ iff there is an ideal dialogue tree $\mT(\delta)$ for $\phi$.
    \end{itemize}   
\end{corollary}

\section{Summary and Conclusion}
We introduce a generic framework to provide a flexible environment for logic argumentation, and to address the challenges of explaining inconsistency-tolerant reasoning. 
Particularly, we studied how deductive arguments, DeLP, ASPIC/ ASPIC+ without preferences, flat or non-flat ABAs and sequent-based argumentation are instances of P-SAF frameworks. 
 (Detailed discussions can be found after Definition~\ref{def:ab-arg} and~\ref{def:ab-att}).
However, different perspectives were considered as follows.

Regarding deductive arguments and DeLP, our work extends these approaches in several ways. First, the usual conditions of minimality and consistency of supports are dropped. This offers a simpler way of producing arguments and identifying them. Second, like ABAs, the P-SAF arguments are in the form of tree derivations to show the structure of the arguments. This offer aims to (1) clarify the argument structure, and (2) enhance understanding of intermediate reasoning steps in inconsistency-tolerant reasoning in,  for instance, \datalogPM and DL.

Similar to ”non-flat” ABAs, the P-SAF framework uses the notion of $\cn$ to allow the inferred assumptions being conflicting. In contrast, "flat" ABAs ignore the case of the inferred assumptions being conflicting. Moreover, by using collective attacks, the P-SAF framework is sufficiently general to model n-ary constraints, which are not yet addressed in ”non-flat” ABAs and ASPIC/ ASPIC+ without preferences. 
Like our approach, contrapositive ABAs in~\cite{HEYNINCK2020103,ArieliH24} provide an abstract view for logical argumentation, in which attacks are defined on the level assumptions. However, since a substantial part of the development of the theory of contrapositive ABA is focused on contrapositive propositional logic, we have considered the logic of ABA as being given by $\cnb_s$ and these contrapositive ABAs being simulated in our setting, see  Section~\ref{subsec:relation-framework}.
In Section~\ref{subsec:relation-framework}, we showed how sequent-based argumentation can fit in the P-SAF setting. While our work can be applied to first-order logic, sequent-based argumentation leaves the study of first-order formalisms for further research.

The work of~\cite{Amgoud2009} proposed the use of Tarski abstract logic in argumentation that is characterized simply by a consequence operator.
However, many logics underlying argumentation systems, like ABA or ASPIC systems, do not always impose the absurdity axiom. 
A similar idea of using consequence operators can be found in the work of~\cite{loanho_2024}.
When a consequence operation is defined by means of "\emph{models}", inference rule steps are implicit within it. If arguments are defined by consequence operators, then the structure of arguments is often ignored, which makes it difficult to clearly explain the acceptability of the arguments. These observations motivate the slight generalizations of Tarski's abstract logic, in which we defined consequence operators in a proof-theoretic manner, inspired by the approach of~\cite{Stephen1975}, with minimal properties.

As we have studied here, we introduced an alternative abstract approach for logical argumentation and showed the connections between our framework and the state-of-the-art argumentation frameworks.
However, we should not claim any framework as better than those, or vice versa. Rather, the choice of an argumentation framework using specific logic should depend on the needs of the application.

Finally, this paper is the first investigation of dialectical proof procedures to compute and explain the acceptance wrt argumentation semantics in the case of collective attacks.
The dialectical proof procedures address the limits of the paper~\cite{loanho_2024}, i.e., it is not easy to understand intermediate reasoning steps in reasoning progress with (inconsistent) KBs.


 The primary message of this paper is that we introduce a generic argumentation framework to address the challenge of explaining inconsistency-tolerant reasoning in KBs. This approach is defined for any logic involving reasoning with maximal consistent subsets, therefore, it provides a flexible environment for logical argumentation. To clarify and explain the acceptance of a sentence with respect to inconsistency-tolerant semantics, we present explanatory dialogue models that can be viewed as dialectic-proof procedures and connect the dialogues with argumentation semantics. The results allow us to provide dialogical explanations with graphical representations of dialectical proof trees. The dialogical explanations are more expressive and intuitive than existing explanation formalisms.
 
 Our approach has been studied from a theoretical viewpoint.
 From practice, especially, from a human-computer interaction perspective, we will perform experiments with our approach in real-data applications. We then qualitatively evaluate our explanation by human evaluation. It would be interesting to analyze the complexity of computing the explanations empirically and theoretically.  


\vskip 0.2in
 \appendix

\section{Preliminaries}
\label{app:pre}

To prove the soundness and complete results, we sketch out \textbf{a general strategy} as follows:

 \begin{enumerate}
    \item Our proof starts with the observation that a dialogue $\delta$ for a formula $\phi$ can be seen as a collection of several (independent) focused sub-dialogues $\delta_1, \ldots, \delta_n$. 
    The dialogue tree $\mT(\delta_i)$ drawn from $\delta_i$ is a subtree of $\mT(\delta)$ drawn from the sub-dialogue $\delta$, and it corresponds to the \emph{abstract dialogue tree} that has root an argument with conclusion $\phi$. 
    (The notion of abstract dialogue tree can be found in~\cite{loanho_2024}).
   Thus it is necessary to consider a \emph{correspondence principle} that links
    dialogue trees to abstract dialogue trees. The materials for this step can be found in Section~~\ref{app:partition} and~\ref{app:DT-AbstractDT}.
   
    \item The correspondence principle allows to utilize the results of abstract dialogue trees in ~[\cite{loanho_2024}, Corollary 1] to prove the soundness results. We extend Corollary 1 of~\cite{loanho_2024} to prove the completeness results.
\end{enumerate}

The proof of the soundness and completeness results depends on some notions and results that we describe next.

\begin{notation} 
\label{not:arg}
Let $\mK$ be a KB, $\mX \subseteq \mK$ be a set of facts and $\mS \subseteq \Arg_{\mK}$ be a set of arguments induced from $\mK$. Then,
\begin{itemize}
    \item $\Args(\mX)  =  \{A \in \Arg_{\mK} \mid \Sup(A)\subseteq \mX \}$ are the set of \emph{arguments generated by} $\mX$,


    \item $\base(\mS)  =  \underset{A\in\mS}{\bigcup}\Sup(A)$ are the set of \emph{supports} of arguments in $\mS$,

    \item An argument $B$ is a \emph{subargument} of argument $A$ iff $\Sup(B)\subseteq \Sup(A)$. We denote the set of subarguments of $A$ as   $\subs(A)$.
\end{itemize}
\end{notation}

\subsection{Abstract Dialogue Trees and Abstract Dialogue Forests}
\label{def:abstract-dia-forest}
 We observe that a formula $\phi$ can have many arguments whose conclusion is $\phi$. 
 Thus a dialogue tree with root $\phi$ can correspond to one, none, or multiple \emph{abstract dialogue trees}, one for each argument for $\phi$.
 We call this set of abstract dialogue trees an \emph{abstract dialogue forest}.
 The following one presents a definition of abstract dialogue forests
 and reproduces a definition of abstract dialogue trees (analogous to Definition 8 in~\cite{loanho_2024}).
 Formally:

\begin{definition} [Abstract dialogue forests]
\label{def:abstract-forests}
Let $ \mAF_ \delta  =  (\Arg_ \delta, \Att_ \delta)$ be the P-SAF drawn from a dialogue $D(\phi)  =  \delta$. An \emph{abstract dialogue forest} (obtained from $ \mAF_ \delta$) for $\phi$ is a set of \emph{abstract dialogue trees}, written $\mF_{\G}(\phi)  =  \{\mT^{1}_{\G}, \ldots, \mT^{h}_{\G}\}$, such that: For each abstract dialogue tree $\mT^{j}_{\G}$ ($j = 1, \ldots, h$), 

\begin{itemize}
    \item the root of $\mT^{j}_{\G}$ is the proponent argument (in $\Arg_ \delta$) with the conclusion $\phi$,

    \item if a node $A$ in $\mT^{j}_{\G}$ is a proponent argument (in $\Arg_ \delta$), then all its children (possibly none) are opponent arguments (in $\Arg_ \delta$) that attack $A$
    

     \item if a node $A$ in $\mT^{j}_{\G}$ is an opponent argument (in $\Arg_ \delta$), then exactly one of the following is true: (1) $A$ has exactly one child, and this child is a proponent argument  (in $\Arg_ \delta$) that attacks $A$; (2) $A$ has more than one child, and all these children are proponent argument  (in $\Arg_ \delta$) that collectively attacks $A$.
     
\end{itemize}   
\end{definition}

\begin{remark} We call the abstract dialogue tree that has root an argument with conclusion $\phi$ an abstract dialogue tree for $\phi$.
\end{remark}


 Fix an abstract dialogue forest $\mF_{\G}(\phi)  =  \{\mT^{1}_{\G}, \ldots, \mT^{h}_{\G}\}$. For such abstract dialogue tree $\mT^{j}_ {\G}$ ($i = 1 , \ldots, h$), we adopt the following conventions:
\begin{itemize}
    \item Let $\mB_1$ be the set of all proponent arguments in $\mT^{j}_{\G}$. $\mDE(\mT^{j}_{\G}) = \{\alpha \mid  \forall A \in \mB_1, \alpha \in \Sup(A) \} \subseteq \mF$ is the \emph{defence set} of $\mT^{j}_{\G}$, i.e. the set of facts in the support of the arguments in $\mB_1$.
   
    \item Similarly, let $\mB_2$ be the set of all opponent arguments in $\mT^{j}_{\G}$. $\mCU(\mT^{j}_{\G}) = \{\beta \mid  \forall B \in \mB_2, \beta \in \Sup(B) \} \subseteq \mF$ is the \emph{culprit set} of $\mT^{j}_{\G}$, i.e. the set of facts in support of the arguments in $\mB_2$. 



\end{itemize}

 We reproduce a definition of \emph{admissible abstract dialogue trees} given in~\cite{loanho_2024}. This notion will be needed for Section~\ref{app:DT-AbstractDT}.

\begin{definition} [Admissible abstract dialogue trees]
\label{def:adm-ab-dt}
An abstract dialogue tree for $\phi$ is said to \emph{admissible} iff the proponent wins and no argument labels both a proponent and an opponent node. 
\end{definition}

Intuitively, in an abstract dialogue tree, a proponent wins if either the tree ends with arguments labelled by proponent nodes or every argument labelling an opponent node has a child.

\subsection{Partitioning a Dialogue Tree into Focused Substrees}
\label{app:partition}
This section shows how to partition a dialogue tree $\mT(\delta)$ into focused subtrees of $\mT(\delta)$. We will need this result to prove soundness and completeness results.

We observe that a dialogue $D(\phi) = \delta$, from which the dialogue tree $\mT(\delta)$ is drawn, may contain one, none, or multiple \emph{focused sub-dialogues} $\delta_i$ of $\delta$. Each dialogue tree $\mT(\delta_i)$ drawn from the focused sub-dialogue $\delta_i$ is a subtree of $\mT(\delta)$ and focused. This is proven in the following lemma.

\begin{lemma}
     \label{lem:DT-subT} 
     Let $\mT(\delta)$ be a dialogue tree (with root $\phi$) drawn from a dialogue $D(\phi) = \delta$. 
     Every focused subtree of $\mT(\delta)$ with root~$\phi$ is the dialogue tree
     drawn from a focused sub-dialogue $\delta_i$ of $\delta$.
\end{lemma}

\begin{proof}
All the subtrees considered in this proof are assumed to have root $\phi$.
The proof proceeds as follows:
\begin{enumerate}  
    \item First, we construct the set of focused dialogue subtrees of $\mT(\delta)$
    \item Second, we show that each focused subtree $\mT(\delta_i)$ of $\mT(\delta)$ is drawn from a focused sub-dialogue $\delta_i$ of $\delta$.
\end{enumerate}

Let $\mT(\delta_1), \ldots, \mT(\delta_m)$ be all the focused subtrees contained
in $\mT(\delta)$ with root~$\phi$. Each focused subtree is obtained in the following way:
First, choose a single utterance at the root and discard the subtrees corresponding to the other utterances at the root. Then proceed (depth first) and for each potential argument
labelled $\op$ with children labelled $\po$, choose a single identifier and select
among them those (and only those) with that identifier; discard the other children labelled $\po$ (which have a different identifier) and the corresponding subtrees.
By Definition~\ref{def:t-focused}, every focused subtree can be obtained in this way.

2. We next prove (2) by showing the construction of the focused sub-dialogue $\delta_i$ that draws $\mT(\delta_i)$.

 Given $m$ dialogue trees $\mT(\delta_1), \ldots, \mT(\delta_m)$ constructed from $\mT(\delta)$, the focused sub-dialogue $\delta_i$ ($1 \leq i \leq m$) drawing the dialogue tree $\mT(\delta_i)$ is constructed as follows:

\begin{itemize}
    \item $\delta_i$ is initialised to empty;
    \item for each node $\psi, [ \_ , \_ , \id]) = N $ in $\mT(\delta_i)$,
    \begin{itemize}
        \item if $u_{id} = (\_ ,\ \tg,\ \_ ,\ \id )$ is in $\delta$ but not in $\delta_i$, then add $u_{id}$ to $\delta_i$;
        \item let $u_{\tg}$ be the utterance in $\delta$; if $u_{\tg}$ is the target utterance of $u_{\id}$, then add $u_{\tg}$ to $\delta_i$;
    \end{itemize}
    \item Sort $\delta_i$ in the order of utterances $ID$.
\end{itemize}

It is easy to see that each $\delta_i$ constructed as above is a focused sub-dialogue of $\delta$ (by the definition of focused sub-dialogues in Definition~\ref{def:dia-tree-DLAF}), and $\mT(\delta_i)$ is drawn from $\delta_i$. Thus (2) is proved.
\end{proof}

\subsection{Transformation from Dialogue Trees into Abstract Dialogue Trees}
\label{app:DT-AbstractDT}

The following \emph{correspondence principle} allows to
translate dialogue trees into abstract dialogue trees and vice versa.

\begin{remark}
Recall that a dialogue tree for $\phi$
has root $\phi$, while an abstract dialogue tree for $\phi$
has root an argument for $\phi$.
\end{remark}

\begin{theorem} [Correspondence principle] \label{thm:def-abs}
Let $\phi \in \mL$ be a formula. Then:

\begin{enumerate}
    \item  For every defensive and non-redudant dialogue tree $\mT(\delta)$ for $\phi$, there exists an admissible abstract dialogue tree $\mT_{\G}$ for $\phi$ such that $\mDE(\mT_{\G}) \subseteq  \mDE(\mT(\delta))$ and $ \mCU(\mT_{\G})  \subseteq  \mCU(\mT(\delta))$.

    \item For every admisslbe abstract dialogue tree $\mT_{\G}$ for $\phi$, there exists a defensive and non-redundant dialogue tree $\mT(\delta)$ for $\phi$ such that $\mDE(\mT(\delta)) \subseteq \mDE(\mT_{\G})$ and $\mCU(\mT(\delta)) \subseteq \mCU(\mT_{\G})$.
\end{enumerate}
\end{theorem}

\begin{proof}
We prove the theorem by transforming dialogue trees into abstract dialogue trees and vice versa.

\textbf{1. The transformation from dialogue trees into abstract dialogue trees}

Given a defensive and non-redundant dialogue tree $\mT(\delta)$ with root $\phi$, its equivalent abstract dialogue tree $\mT_{\G}$ for an argument for $\phi$ in $\mT_{\G}^{1}, \ldots, \mT_{\G}^{h}$ is constructed inductively as follows:

\begin{enumerate}
    \item Modify $\mT(\delta)$ by adding a new \emph{flag} ($0$ or $1$) to  nodes in $\mT(\delta)$ and initialise $0$ for all nodes; a node looks like $(\_, [\_, \_, \_])-0$. The obtained tree is $\mT^{\prime}(\delta)$.

    \item $\mT_{\G}$ is $\mT_{\G}^{h}$ in the sequence $\mT_{\G}^{1}, \ldots, \mT_{\G}^{h}$ constructed inductively as follows:
    \begin{enumerate} [a)] 
    
    \item Let $A$ be the potential argument drawn from $\mT(\delta)$ that contains root $\phi$. $\mT_{\G}^{1}$ contains exactly one node that holds $A$ and is labelled by $\po$. Set the nodes in $\mT^{\prime}(\delta)$ that are in $A$ to $1$. The obtained tree is $\mT^{\prime}_{1}(\delta)$.

    \item Let $\mT_{\G}^{k}$ be the $k$-th tree, with $1 \leq k \leq h$. $\mT^{k+1}_{\G}$ is expanded from $\mT^{k}_{\G}$ by adding nodes $(\tL :\ B_j)$ with $\tL \in \{\po, \op\}$.
    
    For each node $(\tL :\ B_j)$, $B_j$ is a potential argument drawn from $\mT^{\prime}_{k}(\delta)$, which is a child of $C$ - another potential argument drawn from $\mT^{\prime}_{k}(\delta)$, such that:

    \begin{itemize}
        \item there is at least one node in $B_j$ that is assigned $0$;
        \item the root of $B_j$ has a parent node $t$ in $\mT^{\prime}_{k}(\delta)$ such that the flag of $t$  is $1$ and $t$ is in $C$;
        \item if the root of $B_j$ is labelled by $\po$, then $\tL$ is $\po$. Otherwise, $\tL$ is $\op$.
        \item set all nodes in $\mT^{\prime}_{k}(\delta)$ that are also in $B_j$ to 1. The obtained tree is $\mT^{\prime}_{k+1}(\delta)$.
    \end{itemize}
    \end{enumerate}
    \item $h$ is the smallest index s.t there is no node in $\mT^{\prime}_{h}(\delta)$ where its flag is $0$.
\end{enumerate}

$\mT_{\G}$ is constructed as follows:

\begin{itemize}
    \item Every node of $\mT_{\G} = \mT_{\G}^{h}$ includes a potential argument. For each potential argument, there is a unique node in $\mT_{\G}$. Each node is labelled $\po$ or $\op$ as potential arguments drawn from $\mT(\delta)$ are labelled either $\po$ or $\op$. 

    \item  The root of $\mT_{\G}$ includes the potential argument for $\phi$ of the dialogue and labelled $\po$ by constructing $\mT(\delta)$.

    \item  Since $\mT(\delta)$ is defensive, by Definition~\ref{def:defensive-tree}, it is focused and patient.
    Thus there is only one way of attacking a potential argument labelled by $\op$.  
    Since $\mT(\delta)$ is defensive, by Definition~\ref{def:defensive-tree}, it is last-word.
    Then there is no un-attacked (potential) argument labelled by $\op$.
    From the above, it follows that every $\op$ node has exactly one $\po$ node as its child.

    \item Since $\mT(\delta)$ is non-redundant, by Definition~\ref{def:non-re}, no potential argument labels $\op$ and $\po$.    
\end{itemize}

Recall that, in $\mT_{\G}$, the potential arguments labelling $\po$ ($\op$, respectively) are called proponent arguments (opponent arguments, respectively). 

It can be seen that $\mT_{\G}$ has the following properties:
\begin{itemize}
    \item the root is the proponent argument for $\phi$.
    \item the $\op$ node has either exactly one child holding one proponent argument that attacks it or children holding the proponent arguments that collectively attack it.
    \item all leaves are nodes labelled by $\po$, namely $\po$ wins and there is no node labelled by both $\po$ and $\op$.
\end{itemize}

It follows immediately that $\mT_{\G}$ is an admissible abstract dialogue tree (by the definition of abstract dialogue trees in Definition~\ref{def:abstract-forests} and Definition~\ref{def:adm-ab-dt}).

Since $\mT_{\G}$ contains the same potential arguments as $\mT(\delta)$ and the arguments have the same $\po / \op$ labelling in both $\mT_{\G}$ and $\mT(\delta)$, we have $\mDE(\mT_{\G})  =  \mDE(\mT(\delta))$ and $ \mCU(\mT_{\G})  =  \mCU(\mT(\delta))$.

\textbf{2. The transformation from abstract dialogue trees into dialogue trees}

We first need to introduce some new concepts that together constitute the dialogue tree.

\begin{definition} [Support trees]
A \emph{support tree} of a formula $\alpha$ is defined as follows:
\begin{enumerate}
    \item The root is a proponent node labelled by $\alpha$.
    \item Let $N$ be a proponent node labelled by $\sigma$. If $\sigma$ is a fact, then either $N$ has no children, or $N$ has children that are opponent nodes labelled by $\beta_k$, $k = 1, \ldots, n$ such that $\{ \beta_k \} \cup \{ \sigma \}$ is inconsistent.
     If $\sigma$ is a non-fact, then one of the following holds:
     \begin{itemize}
         \item either (1) $N$ has children that are proponent nodes labelling $\omega_l$, $l = 1, \dots, m$, such that $\sigma \in \cn(\{ \omega_l \})$,
         \item or (2) $N$ has children that are opponent nodes labelling $\beta_k$, $k = 1, \ldots, n$, such that $\{ \beta_k \} \cup \{ \sigma \}$ is inconsistent,
         \item or both (1) and (2) hold.
     \end{itemize}
\end{enumerate}   
\end{definition}

\begin{definition} [Context trees]
A \emph{context tree} of a formula $\alpha$ is defined as follows:
\begin{enumerate}
    \item The root is an opponent node labelled by $\alpha$.
    \item Let $N$ be an opponent node labelled by $\sigma$. If $\sigma$ is a fact, then $N$ has children that are proponent nodes labelled by $\beta_k$, with $k = 1, \ldots, n$, such that $\{ \beta_k \} \cup \{ \sigma \}$ is inconsistent and the children have the same identify ~\footnote{The condition of "children having the same identify" ensure that a potential argument is attacked by exactly one potential argument if $k = 1$ or collectively attacked by one set of potential arguments if $k > 1$.}.
     If $\sigma$ is a non-fact, then one of the following holds:
     \begin{itemize}
         \item either (1) $N$ has children that are opponent nodes labelled by $\omega_l$, with $l = 1, \dots, m$, such that $\sigma \in \cn(\{ \omega_l \})$,
         \item or (2) $N$ has children that are proponent nodes labelling $\beta_k$, with $k = 1, \ldots, n$, such that $\{ \beta_k \} \cup \{ \sigma \}$ is inconsistent and the children have the same identify,
         \item or both (1) and (2) hold.
     \end{itemize}
\end{enumerate}
   
\end{definition}

Now we prove the translation from abstract dialogue trees to dialogue trees.

Given an abstract dialogue tree $\mT_{\G}$ for $\phi$, its equivalent focused subtree $\mT(\delta_i)$ of a dialogue tree $\mT(\delta)$ for $\phi$ is constructed inductively as follows:

Let $\mDE(\mT_{\G})$ and $\mCU(\mT_{\G})$ be the defence set and the set of culprits of $\mT_{\G}$ , respectively.

For each $\alpha \in \mCU(\mT_{\G})$, let $arg(\beta_k)$, with $k = 1, \ldots, n$, be a set of arguments for $\beta_k$ labelling nodes in $\mT_{\G}$ such that $\{ \beta_k \} \cup \{\alpha\}$ is inconsistent.
If $k = 1$, then there exists a single argument for $\beta$ that attacks an argument including $\alpha$. We say that an argument including $\alpha$ is an argument whose conclusion or support includes $\alpha$.
If $ k > 1$, then there exists a set of arguments for $\beta_k$ that collectively attacks an argument including $\alpha$.

For each $\alpha \in \mDE(\mT_{\G})$, let $B^{\alpha}_{k}$ be the set of facts in the support of arguments for $\beta_k$ in $\mT_{\G}$ such that the arguments for $\beta_k$ attack an argument including $\alpha$.
Clearly, there exists an argument for $\beta$ that attacks an argument including $\alpha$ if $k = 1$ and there exist arguments for $\beta_k$ that attack or collectively attack an argument including $\alpha$ if $ k > 1$.

We construct inductively the sequence of trees $\mT^{1}(\delta_i), \ldots, \mT^{h}(\delta_i)$ as follows:
\begin{enumerate}
    \item $\mT^{1}(\delta_i)$ is a support tree of $\phi$ corresponding to the argument labelling the root of $\mT_{\G}$.

    \item  
    Let $j = 2n$ such that the non-terminal nodes in the frontier of  $\mT^{j}(\delta_i)$ are opponent nodes labelled by a set of formulas $\beta_k$ where $\{\beta_k \} \cup \{ \alpha \}$ is inconsistent and  $\alpha \in \mDE(\mT_{\G})$.
    Expand each such node by a context tree of $\beta_k$ wrt $B^{\alpha}_{k}$. The obtained tree is $\mT^{j+1}(\delta_i)$.

    \item  
    Let $j = 2n + 1$ such that the non-terminal nodes in the frontier of $\mT^{j}(\delta_i)$ are proponent nodes labelled by a set of formula $\beta_k$, where $\{\beta_k \} \cup \{ \alpha \}$ is inconsistent and $\alpha \in \mCU(\mT_{\G})$.
    Expand each such node by a support tree of $\beta_k$ corresponding to the argument for $\beta_k$. The obtained tree is $\mT^{j+1}(\delta_i)$.

    \item Define $\mT(\delta_i)$ to be the limit of $\mT^{j}(\delta_i)$.  
\end{enumerate}

It follows immediately that $\mT(\delta_i)$ is a dialogue tree for $\phi$ whose defence set is a subset of the defence set of $\mT_{\G}$ and whose the culprit set is a subset of the culprit set of $\mT_{\G}$.
Since $\mT_{\G}$ is admissible, by Defintion~\ref{def:adm-ab-dt}, the proponent wins, namely either the tree ends with arguments labelled by proponent nodes or every argument labelled by an opponent node has a child. By Definition~\ref{def:defensive-tree}, $\mT(\delta_i)$ is defensive.
Since $\mT_{\G}$ is admissible,  by Defintion~\ref{def:adm-ab-dt}, $\mT_{\G}$ has no argument labelling both a proponent and an opponent node. By Definition~\ref{def:non-re}, $\mT(\delta_i)$ is non-redundant.

\end{proof}

\subsection{Notions and Results of Acceptance of an Argument from Its Abstract Dialogue Trees}

For reader's convenience, we reproduce here definitions and results for abstract dialogues for $\phi$ that can also be found in~\cite{loanho_2024}. 
We use a similar argument for focused sub-dialogues of a dialogue $\delta$ for $\phi$. 
In fact, the definitions and the results we reproduce here are essentially the same with abstract dialogues replaced by focused sub-dialogues of $\delta$ for $\phi$. This replacement is because an abstract dialogue for $\phi$ can be seen as a focused sub-dialogue of a dialogue $\delta$ for $\phi$. This follows immediately from the results in Lemma~\ref{lem:DT-subT} 
(i.e., showing a dialogue tree drawn from a dialogue $\delta$ for $\phi$ can be divided into focused sub-trees drawn from focused sub-dialogues of $\delta$ for $\phi$)
and in Theorem~\ref{thm:def-abs} 
(i.e., showing each such focused subtrees corresponds with an abstract dialogue tree drawn from an abstract dialogue).

\begin{definition} [Analogous to Definition 9 in~\cite{loanho_2024}]
\label{def:analogous-Def9}
Let $\mT_{\G}$ be the abstract dialogue tree drawn from a focused sub-dialogue $\delta^ \prime$ of a dialogue for $\phi$. The focused sub-dialogue $\delta^ \prime$ is called
    \begin{itemize}
        \item \emph{admissible-successful} iff $\mT_{\G}$ is admissible;

        \item \emph{preferred-successful} iff it is admissible-successful;
        
        \item  \emph{grounded-successful} iff $\mT_{\G}$ is admissible and finite;

        \item  \emph{sceptical-successful} iff $\mT_{\G}$ is admissible and for no opponent node in it, there exists an admissible dialogue tree for the argument labelling an opponent node.     
    \end{itemize} 
\end{definition}



     
    

\begin{corollary} [Analogous to Corollary 1 in~\cite{loanho_2024}]
\label{cor:analogous-Cor1}
Let $\mK$ be a KB, $\phi \in \mL$ a formula and $\mAF_{\mK}$ be the corresponding P-SAF of $\mK$. Then, $\phi$ is
\begin{itemize}
    \item   credulously accepted in some admissible/preferred extension of $\mAF_{\mK}$ if there is a focused sub-dialogue $\delta^{\prime}$ of a dialogue for $\phi$ such that $\delta^{\prime}$ is admissible/preferred-successful;
    
    \item groundedly accepted in a grounded extension of $\mAF_{\mK}$ if there is a focused sub-dialogue $\delta^{\prime}$ of a dialogue for $\phi$  such that $\delta^{\prime}$ is grounded-successful;

    \item  sceptically accepted in all preferred extensions of $\mAF_{\mK}$ if there is a focused sub-dialogue $\delta^{\prime}$ of a dialogue for $\phi$ such that $\delta^{\prime}$ is sceptical-successful.   
\end{itemize} 
\end{corollary}

\section{Proofs for Section~\ref{sec:soundness}} 
\label{app:proof-soundness}

We follow the general strategy to prove Theorem~\ref{thm:adm},~\ref{thm:prf-stb}, ~\ref{thm:ground} and~\ref{thm:scep}. In particular, we:

\begin{enumerate}
    \item partition a dialogue tree for $\phi$ drawn from a dialogue $\delta$ into subtrees for $\phi$ drawn from the focused sub-dialogue of $\delta$ (by Lemma~\ref{lem:DT-subT}),
    \item use the correspondence principle to transfer each subtree for $\phi$ into an abstract dialogue tree for $\phi$ (by Theorem~\ref{thm:def-abs}),
    \item apply Definition~\ref{def:analogous-Def9} (analogous to Definition 9 in~\cite{loanho_2024} and Corollary~\ref{cor:analogous-Cor1} (analogous to Corollary 1 in~\cite{loanho_2024}) for abstract dialogue trees to prove the soundness results. 
\end{enumerate}


    

  
\subsection{Proof of Theorem~\ref{thm:adm}}

\thmcredulous*

  \begin{proof}

  Let $\mT(\delta)$ be a dialogue tree drawn from a dialogue $\delta$ for $\phi$. Let $\mT(\delta_1), \ldots , \mT(\delta_m)$ be sub-trees (with root $\phi$) constructed from $\mT(\delta)$. By Lemma~\ref{lem:DT-subT}, we know all sub-trees $\mT(\delta_i)$ of $\mT(\delta)$ are the dialogue trees drawn from focused sub-dialogues $\delta_i$, with $i = 1, \ldots, m$, of $\delta$ and each such sub-tree is focused. We assume that $\mT(\delta_i)$ is defensive and non-redundant.
%
  By the correspondence principle of Theorem~\ref{thm:def-abs},  there is an abstract dialogue tree $\mT_{\G}^{i}$ for $\phi$ such that $\mDE(\mT_{\G}^{i})  =  \mDE(\mT(\delta_i))$ and $\mCU(\mT_{\G}^{i})  =  \mCU(\mT(\delta_ i))$.
  
  For such abstract dialogue tree, let $\mB$ be a set of proponent arguments in $\mT_{\G}^{i}$, we prove that 
  \begin{enumerate}
      \item $\mB$ attacks every attack against it and $\mB$ does not attack itself;
      \item no argument in $\mT_{\G}^{i}$ that labels both $\po$ or $\op$.
  \end{enumerate}

 Since $\mT(\delta_i)$ is defensive and non-redundant, we get the following statements:
  
  \begin{enumerate} [a)]
        \item By Definition~\ref{def:defensive-tree}, it is focused. Then the root of $\mT_{\G}^{i}$ hold a proponent argument for $\phi$.
        
      \item By Definition~\ref{def:defensive-tree}, it is last-word. Then the proponent arguments in $\mT_{\G}^{i}$ are leaf nodes, i.e., $\po$ wins. Thus, $\mB$ attacks every attack against it.

      \item By Definition~\ref{def:defensive-tree}, $\mT(\delta_{i})$ has no formulas $\alpha_h$ in opponent nodes belong to $\mDE(\mT(\delta_i))$ such that $\{ \alpha_h \} \cup \mDE(\mT(\delta_i))$ is inconsistent. 
      $h = 0$ corresponds to the case that there is no potential argument, say $A$, such that $\{ \alpha \}$ is the support of $A$ and $A$ attacks any potential arguments supported by $\mDE(\mT(\delta_i))$. 
      Similarly, for $h > 0$, there are no potential arguments collectively attacking any potential arguments supported by $\mDE(\mT(\delta_i))$.
      Thus, $\mB$ does not attack itself.

      \item By Definition~\ref{def:non-re} of non-redundant trees, there is no argument in $\mT_{\G}^{i}$ that labels both $\po$ or $\op$.
  \end{enumerate}

  It can be seen that (b) and (c) prove (1), and (d) proves (2). 
  It follows that $\mT_{\G}^{i}$ is admissible. This result directly follows from the definition of admissible abstract dialogue trees in Definition~\ref{def:adm-ab-dt} (analogous to those in~\cite{loanho_2024}).  
  By Definition~\ref{def:analogous-Def9} (analogous to Definition 9 in~\cite{loanho_2024}), $\delta_i$ is admissible-successful in the P-SAF framework $\mAF_{\delta_{i}}$ drawn from $\delta_i$ (supported by $\mDE(\mT(\delta_i))$).

  Now, we need to show $\delta$ is admissible-successful. By Definition~\ref{def:abstract-dia-forest}, each tree in the abstract dialogue forest contains its own set of proponent arguments, namely, $\mDE(\mT_{\G}^{i}) \neq \mDE(\mT_{\G}^{l})$, with $1 \leq i,\ l \leq m,\ i \neq l$. Thus, arguments in other trees do not affect arguments in $\mAF_{\delta_{i}}$. It follows that $\delta$ is admissible-successful. 
  By Corollary~\ref{cor:analogous-Cor1} (analogous to Corollary 1 in~\cite{loanho_2024}), $\phi$ is credulously accepted under $\adm$ semantics in $\mAF_{\delta}$.
\end{proof}



 \subsection{Proof of Theorem~\ref{thm:ground}}
The following lemma is used in the proof of Theorem~\ref{thm:ground}
\begin{lemma}
\label{lem:inf-groundedset}
   An abstract dialogue tree $\mT_{\G}$ is finite iff the set of arguments labelling $\po$ in $\mT_{\G}$ is a subset of the grounded set of arguments.
\end{lemma}

\begin{proof} Let $\mB$ is a set of arguments labbeling $\po$ in$\mT_{\G}$. We prove the lemma as follows:

If part: The height of a finite dialogue tree $\mT_{\G}$ is $2h$. We prove that $\mB$ is a subset of the grounded set by induction on $h$. Observer that (1) if $\mA_1$, $\mA_2$ are an admissible subset of the grounded set, then so is $\mA_1 \cup \mA_2$; (2) if an argument $A$ is accepted wrt $\mA_1$ then $\mA_1 \cup \{A\}$ is an admissible subset of the grounded set. 

$h = 0$ corresponds to dialogue trees containing a single node labelled by an argument that is not attacked. So $\mB$ containing only that argument is a subset of the grounded set. 

Assume the assertion holds for all finite dialogue tree of height smaller than $2h$. Let $A$ be an argument labelling $\po$ at the root of $\mT_{\G}$. For each argument $B$ attacking a child of $A$, let $\mT_{B}$ be the subtree of $\mT_{\G}$ rooted at $B$. Clearly, $\mT_{B}$ is a finite dispute tree with height smaller than $2h$. The union of sets of arguments labelling $\po$ of all $\mT_{B}$ is a subset of the grounded set and defends $A$. So $\mB$ is a subset of the grounded set.

Only if part: to construct a finite abstract dialogue tree for a groundedly accepted argument in a finite P-SAF framework, we need the following lemma: In a finite P-SAF framework, the grounded set equals $\emptyset \cup \mF(\emptyset) \cup \mF^{2}(\emptyset) \cup \cdots$. This lemma follows from two facts, proven in~~\cite{Dung95}, of the characteristic function $\mF$:
\begin{itemize}
    \item $\mF$ is monotonic w.r.t. set inclusion.
    \item if the argument framework is finite, then $\mF$ is $\omega-$ continuous.
\end{itemize}
For each argument $A$ in the grounded set, $A$ can be ranked by a natural number $r(A)$ such that $A \in \mF^{n(A)} (\emptyset) \setminus \mF^{r(A) -1} (\emptyset)$. So $r(A) = 1$, then $A$ belongs to the grounded set and is not attacked. $r(A) = 2$, the set of arguments (of the grounded set) defended by the set of arguments (of the grounded set) such that $r(A) \leq 1$ and so on. For each set of arguments $\mS$ collectively attacking the grounded set,  the rank of $\mS$ is $\texttt{min} \{r(A) \mid A \text{ is in the grounded set and attacks some argument in } \mS \}$. Clearly, $\mS$ does not attack any argument in the grounded set of rank smaller than the rank of $\mS$.

Given any argument $A$ in the grounded set, we can build an abstract dialogue tree $\mT_{\G}$ for $A$ as follows: The root of $\mT_{\G}$ is labelled by $A$. For each set of arguments $\mS$ attacking $A$, we select a set of arguments $\mC$ to counterattacks $\mS$ such that the rank of $\mC$ equals the rank of $\mS$, then for each set of arguments $\mE$ attacking some arguments of $\mC$, we select arguments $\mF$ to counterattacks $\mE$ such that the rank of $\mF$ equals to the rank of $\mE$, and so on. So for each branch of $\mT_{\G}$, the rank of a proponent node is equal to that of its opponent parent node, but the rank of an opponent node is smaller than that of its parent proponent node. Clearly, ranking decreases downwards. So all branches of $\mT_{\G}$ are of finite length. Since the P-SAF is finite, $\mT_{\G}$ is finite in breath. Thus  $\mT_{\G}$ is finite.

\end{proof}


\thmground*

\begin{proof}
Let $\mT(\delta)$ be a dialogue tree drawn from a dialogue $\delta$ for $\phi$ and $\mT(\delta_1), \ldots , \mT(\delta_m)$ (with root $\phi$) be subtrees of $\mT(\delta)$. By Lemma~\ref{lem:DT-subT}, all sub-trees of $\mT(\delta)$ are the dialogue trees drawn from focused sub-dialogues $\delta_i$, with $i = 1, \ldots, m$, of $\delta$ and each such subtree is focused. Assume that $\mT(\delta_i)$ is defensive and finite.
Since $\mT(\delta_i)$ is defensive,
by the correspondence principle of Theorem~\ref{thm:def-abs}, there is an abstract dialogue tree $\mT_{\G}^{i}$ for $\phi$ such that $\mDE(\mT_{\G}^{i})  =  \mDE(\mT(\delta_i))$ and $ \mCU(\mT_{\G}^{i})  =  \mCU(\mT(\delta_ i))$.

For such abstract dialogue tree, let $\mB$ be the set of arguments labelling $\po$ in  $\mT_{\G}^{i}$. We need to show that 

\begin{enumerate}
    \item $\mT_{\G}^{i}$ is finite;
    \item $\mB$ attacks ever attract against it, and
    \item $\mB$ does not attack itself.
\end{enumerate}

 Similar to the proof of Theorem~\ref{thm:adm}, (2) and (3) directly follows from the fact that $\mT(\delta_i)$ is admissible.
 Trivially, every $\mT(\delta_i)$ is finite, then $\mT_{\G}^{i}$ is finite.
 As a direct consequence of Lemma~\ref{lem:inf-groundedset}, we obtain that $\mB$ is a subset of the grounded set of arguments in $\mT_{\G}^{i}$.
 By Definition~\ref{def:analogous-Def9} (analogous to Definition 9 in~\cite{loanho_2024}), $\delta_i$ is grounded-successful in $\mAF_{\delta_i}$ (drawn from $\delta_i$).

We next prove that $\delta$ is grounded-successful. Since $\delta_i$ is grounded-successful in $\mAF_{\delta_{i}}$ drawn from $\delta_i$, 
it follows that there are no arguments attacking the arguments in $\mB$ that have not been counter-attacked in the abstract dialogue forest $\mT_{\G}^{1}, \ldots, \mT_{\G}^{m}$ (obtained from the P-SAF drawn from $\delta$).
%
%
By Definition~\ref{def:abstract-dia-forest}, each tree in the abstract dialogue forest contains its own set of proponent arguments, namely, $\mDE(\mT_{\G}^{i}) \neq \mDE(\mT_{\G}^{l})$, with $1 \leq i,\ l, \leq m,\ i \neq l$. If the set of proponent arguments in $\mT_{\G}^{i}$ drawn from the focused sub-dialogue $\delta_i$ is grounded, it is also grounded in $\mAF_{\delta}$ drawn from $\delta$. Thus, $\delta$ is grounded successful. By Corollary~\ref{cor:analogous-Cor1} (analogous to Corollary 1 in~\cite{loanho_2024}), $\phi$ is groundedly accepted in $\mAF_{\delta}$ (supported by $\mDE(\mT(\delta_i))$).
\end{proof}
  
\subsection{Proof of Theorem~\ref{thm:scep}}

\thmsceptical*

 \begin{proof} We prove this theorem for the case of admissible semantics. The proof for preferred (stable) semantics is analogous.
 
  Let $\mT(\delta)$  be a dialogue tree drawn from a dialogue $\delta$ for $\phi$, and $\mT(\delta_1), \ldots , \mT(\delta_m)$ (with root $\phi$) be subtrees constructed from $\mT(\delta)$. By Lemma~\ref{lem:DT-subT}, we know all subtrees in $\mT(\delta)$ are dialogue trees drawn from focused sub-dialogues $\delta_i$, ($i = 1, \ldots , m$) of $\delta$ and each tree is focused. Assume that $\mT(\delta)$ is ideal.   
  Since $\mT(\delta)$ is ideal, $\mT(\delta_i)$ is ideal. By Definition~\ref{def:tree-ideal}, $\mT(\delta_i)$ is defensive.
  By the correspondence principle, there is an abstract dialogue tree $\mT_{\G}^{i}$ for $\phi$ such that $\mDE(\mT_{\G}^{i})  =  \mDE(\mT(\delta_i))$ and $ \mCU(\mT_{\G}^{i})  =  \mCU(\mT(\delta_ i))$.

  For such abstract dialogue tree $\mT_{\G}^{i}$, we need to show that:
  \begin{enumerate}
      \item $\mT_{\G}^{i}$ is admissible, and
      \item for no opponent node $\op$ in it there exists an admissible dialogue tree for the opponent argument.
  \end{enumerate} 
  
  Similar to the proof of Theorem~\ref{thm:adm}, (1) holds.
  Since $\mT(\delta)$ is ideal, by Definition~\ref{def:tree-ideal}, there is a dialogue tree $\mT(\delta_i)$ such that none of the opponent arguments drawn from $\mT(\delta_{i})$ belongs to an admissible set of arguments in $\mAF_{\delta}$ drawn from $\delta$. Then, (2) holds.
  
  We have shown that $\mT_{\G}^{i}$ is defensive and none of the opponent arguments belongs to an admissible set of arguments in $\mAF_{\delta}$ drawn from $\delta$. By Definition~\ref{def:analogous-Def9} (analogous to Definition 9 in~\cite{loanho_2024}), $\delta$ is sceptical-successful. By Corollary~\ref{cor:analogous-Cor1} (analogous to Corollary 1 in~\cite{loanho_2024}), $\phi$ is sceptically accepted in $\mAF_{\delta}$.
  \end{proof}

\section{Proofs for Section~\ref{sec:completeness}}
\label{app:proof-completeness}

To prove the completeness results, we

\begin{enumerate}
    \item partition a dialogue $\delta$ for $\phi$ into its focused sub-dialogues of $\delta$ (by Definition~\ref{def:focused-sub-dia}),

    \item apply Corollary~\ref{cor:extend-cor2} and Definition~\ref{def:analogous-Def9} (analogous to Definition 9 in~\cite{loanho_2024}) to obtain the existence of an abstract dialogue tree drawn from each focused sub-dialogue of $\delta$ wrt argumentation semantics,

    \item use the correspondence principle to transfer from each abstract dialogue tree for $\phi$ into a dialogue tree for $\phi$ (by Theorem~\ref{thm:def-abs}), thereby proving the completeness results.
    
\end{enumerate}

\subsection{Preliminaries}

Corollary~\ref{cor:analogous-Cor1} is used for the proof of the soundness results. To prove the completeness result, we need to extend Corollary~\ref{cor:analogous-Cor1} as follows:

\begin{corollary}
\label{cor:extend-cor2}

Let $\mK$ be a KB, $\phi \in \mL$ a formula and $\mAF_{\mK}$ be the corresponding P-SAF of $\mK$. We say that if $\phi$ is

\begin{itemize}
    \item   credulously accepted in some admissible/preferred extension of $\mAF_{\mK}$, then there is a focused sub-dialogue $\delta_i$, with $i = 1, \ldots, m$, of a dialogue for $\phi$ such that $\delta_i$ is admissible/preferred-successful;
    
    \item groundedly accepted in a grounded extension of $\mAF_{\mK}$, then there is a focused sub-dialogue $\delta_i$ of a dialogue for $\phi$  such that $\delta_i$ is grounded-successful;

    \item  sceptically accepted in all preferred extensions of $\mAF_{\mK}$, then there is a focused sub-dialogue $\delta_i$ of a dialogue for $\phi$ such that $\delta_i$ is sceptical-successful.   
\end{itemize}     
\end{corollary}

 \begin{proof} 
 Let $\mT(\delta)$ be a dialogue tree drawn from a dialogue $\delta$ for $\phi$ and $\mT(\delta_1), \ldots , \mT(\delta_m)$ (with root $\phi$) be subtrees of $\mT(\delta)$. By Lemma~\ref{lem:DT-subT}, all sub-trees of $\mT(\delta)$ are the dialogue trees drawn from focused sub-dialogues $\delta_i$, with $i = 1, \ldots, m$, of $\delta$ and each such subtree is focused.
 By the correspondence principle, there is an abstract dialogue tree $\mT_{\G}^{i}$ for an argument $A$ with conclusion $\phi$, or simply, an abstract dialogue $\mT_{\G}^{i}$ for $\phi$, which corresponds to the dialogue tree $\mT(\delta_i)$ for $\phi$.
 
 It is clear that there is a focused sub-dialogue $\delta_i$  such that it is admissible-successful if $\phi$ is credulously accepted in some admissible/preferred extension of $\mAF_ \mK$.
 The argument given in Definition 2 and Lemma 1 of~\cite{ThangDH09} for binary attacks generalizes to collective attacks, implying that $A$ is accepted in some admissible extension iff $\mT_{\G}^{i}$ is admissible which in turn holds iff $\po$ wins and no argument labels both a proponent and an opponent node.  By Definition~\ref{def:analogous-Def9}, $\delta_i$ is admissible-successful iff $\mT_{\G}^{i}$ is admissible. Thus, the statement is proved.
 
  $\delta_i$ is preferred-successful if $\delta_i$ is admissible-successful.
  This result directly follows from the results of~\cite{Dung95} that states that an extension is preferred if it is admissible. Thus, if $\phi$ is credulously accepted in some preferred extension, then $\delta_i$ is preferred-successful. 
  
  The other statement follows in a similar way as a straightforward generalization of Theorem 1 of~\cite{ThangDH09} for the "grounded-successful" semantic; Definition 3.3 and Theorem 3.4 of~\cite{DUNG2007642} for the "sceptical-successful" semantic.
\end{proof}

\subsection{Proof of Theorem~\ref{thm:com-adm}}

\compadm*

\begin{proof} 
Given $\delta$ be a dialogue for a formula $\phi \in \mL$,
we assume that $\phi$ is credulously accepted under $\adm$ in $\mAF_{\delta}$ drawn from $\delta$ and $\delta$ is admissible-successful.
We prove that there exists a dialogue tree drawn from the sub-dialogue of $\delta$ such that the dialogue tree is defensive and non-redundant.



Let $\delta_1, \ldots, \delta_m$, where $i = 1, \ldots, m$, be sub-dialogues of $\delta$ and each sub-dialogue is focused.
Since $\phi$ is credulously accepted under $\adm$ in $\mAF_{\delta}$ drawn from the dialogue $\delta$, by Corollary~\ref{cor:extend-cor2}, the focused sub-dialogue $\delta_{i}$ is admissible-successful.
By Definition~\ref{def:analogous-Def9} (analogous to Definition 9 in~\cite{loanho_2024}), there is an admissible abstract dialogue tree $\mT^{i}_{\G}$ for $\phi$ drawn from the admissible-successful dialogue $\delta_i$.
By the correspondence principle, there exists a dialogue tree $\mT(\delta_i)$ for $\phi$ that corresponds to the abstract dialogue tree $\mT_{\G}^{i}$ for $\phi$ such that  $\mDE(\mT_{\G}^{i})  =  \mDE(\mT(\delta_i))$ and $ \mCU(\mT_{\G}^{i})  =  \mCU(\mT(\delta_ i))$.
Since $\mT_{\G}^{i}$ is admissible, $\mT(\delta_i)$ is defensive and non-redundant. Thus, the statement is proved.    
\end{proof}

\subsection{Proof of Theorem~\ref{thm:com-ground}}

\compground*

\begin{proof}
Given $\delta$ be a dialogue for a formula $\phi \in \mL$, we assume that $\phi$ is groundedly accepted under $\grd$ in $\mAF_{\delta}$ drawn from $\delta$ and $\delta$ is groundedly-successful. We prove that there exists a dialogue tree drawn from the sub-dialogue of $\delta$ such that the dialogue tree is defensive and finite.

Let $\delta_1, \ldots, \delta_m$, where $i = 1, \ldots, m$, be sub-dialogues of $\delta$ and each sub-dialogue is focused.
Since $\phi$ is groundedly accepted under $\grd$ in $\mAF_{\delta}$ drawn from the dialogue $\delta$, by Corollary~\ref{cor:extend-cor2}, the focused sub-dialogue $\delta_{i}$ is grounded-successful.
By Definition~\ref{def:analogous-Def9} (analogous to Definition 9 in~\cite{loanho_2024}), there is an admissible and finite abstract dialogue tree $\mT^{i}_{\G}$ for $\phi$ drawn from the grounded-successful dialogue $\delta_i$.
By the correspondence principle, there is a dialogue tree $\mT(\delta_i)$ for $\phi$ that corresponds to the abstract dialogue tree $\mT_{\G}^{i}$ for $\phi$ such that  $\mDE(\mT(\delta_i)) = \mDE(\mT_{\G}^{i})$ and $\mCU(\mT(\delta_ i)) = \mCU(\mT_{\G}^{i})$.
Since $\mT_{\G}^{i}$ is admissible, $\mT(\delta_i)$ is defensive. Since $\mT_{\G}^{i}$ is finite, $\mT(\delta_i)$ is finite.  Thus, the statement is proved.
\end{proof}

\subsection{Proof of Theorem~\ref{thm:com-scep}}

\compsceptical*

\begin{proof} We prove this theorem for the case of admissible semantics. The proof for preferred (stable) semantics is analogous.

Given $\delta$ be a dialogue for a formula $\phi \in \mL$, we assume that $\phi$ is credulously accepted under $\adm$ in $\mAF_{\delta}$ drawn from $\delta$ and $\delta$ is admissible-successful. 
We prove that there is an ideal dialogue tree $\mT(\delta)$ for $\phi$ drawn from the dialogue $\delta$.

Let $\delta_1, \ldots, \delta_m$, where $i = 1, \ldots, m$,  be focused sub-dialogues of $\delta$.
%
Since $\phi$ is sceptically accepted under $\adm$ in $\mAF_{\delta}$ drawn from the dialogue $\delta$, by Corollary~\ref{cor:extend-cor2}, the focused sub-dialogue $\delta_{i}$ is sceptical-successful.
By Definition~\ref{def:analogous-Def9} (analogous to Definition 9 in~\cite{loanho_2024}), there is an abstract dialogue tree $\mT^{i}_{\G}$ for $\phi$ drawn from the sceptical-successful dialogue $\delta_i$ such that $\mT^{i}_{\G}$ is admissible and for no opponent node in it there exists an admissible abstract dialogue tree for the argument labelling an opponent node.
Since $\mT^{i}_{\G}$ is admissible, by the correspondence principle, there is a dialogue tree $\mT(\delta_i)$ for $\phi$ that corresponds to the abstract dialogue tree $\mT_{\G}^{i}$ for $\phi$ such that  $\mDE(\mT(\delta_i)) = \mDE(\mT_{\G}^{i})$ and $\mCU(\mT(\delta_ i)) = \mCU(\mT_{\G}^{i})$.

For such the dialogue tree $\mT(\delta_i)$, we need to show that:
\begin{enumerate}
    \item $\mT(\delta_i)$ is defensive and non-redundant,
    \item none of the opponent arguments obtained from $\mT(\delta_i)$ belongs to an admissible set of potential arguments in $\mAF_{\delta_i}$ drawn from $\mT(\delta_i)$.
\end{enumerate}

Since the abstract dialogue $\mT_{\G}^{i}$ is admissible, (1) holds. 
We have that, for no opponent node in $\mT_{\G}^{i}$, there exists an admissible abstract dialogue tree for the argument labelling by $\op$. From this, we obtain that (2) holds.
By Definition~\ref{def:tree-ideal}, $\mT(\delta_i)$ is ideal. Thus, $\mT(\delta)$ is ideal. Thus, the statement is proved.


\end{proof}

\vskip 0.2in

\bibliographystyle{amsplain}

\begin{thebibliography}{10}
\bibitem{Andrea2011}Cali, A., Gottlob, G., Lukasiewicz, T. \& Pieris, A. Datalog+-: A Family of Languages for Ontology Querying. {\em Workshop, Datalog}. (2011)
\bibitem{BAGET20111620}Baget, J., Leclère, M., Mugnier, M. \& Salvat, E. On rules with existential variables: Walking the decidability line. {\em Artificial Intelligence}. (2011)
\bibitem{Alrabbaa2020}Alrabbaa, C., Baader, F., Borgwardt, S., Koopmann, P. \& Kovtunova, A. Finding Small Proofs for Description Logic Entailments: Theory and Practice.  (2020)
\bibitem{ThangDP22}Thang, P., Dung, P. \& Pooksook, J. Infinite arguments and semantics of dialectical proof procedures. {\em Argument Comput.}. \textbf{13}, 121-157 (2022)
\bibitem{Marnette2009}Marnette, B. Generalized Schema-Mappings: From Termination to Tractability. {\em ACM Symposium On Principles Of Database Systems}. (2009)
\bibitem{Dung95}Dung, P. On the Acceptability of Arguments and its Fundamental Role in Nonmonotonic Reasoning, Logic Programming and n-Person Games. {\em Artif. Intell.}. \textbf{77}, 321-358 (1995)
\bibitem{LoanHo2022}Ho, L., Arch-int, S., Acar, E., Schlobach, S. \& Arch-int, N. An argumentative approach for handling inconsistency in prioritized Datalog\(\pm\) ontologies. {\em AI Commun.}. \textbf{35}, 243-267 (2022)
\bibitem{Yun2017GraphTP}Yun, B., Croitoru, M., Vesic, S. \& Bisquert, P. Graph Theoretical Properties of Logic Based Argumentation Frameworks: Proofs and General Results. {\em Proceeding Of GKR}. (2017)
\bibitem{yun2018}Yun, B., Vesic, S. \& Croitoru, M. Toward a More Efficient Generation of Structured Argumentation Graphs. {\em COMMA}. (2018)
\bibitem{AMGOUD20142028}Amgoud, L. Postulates for logic-based argumentation systems. {\em IJAR}. (2014)
\bibitem{Borg2021}Borg, A. \& Bex, F. A Basic Framework for Explanations in Argumentation. {\em IEEE Intelligent Systems}. (2021)
\bibitem{VreeswijkP00}Vreeswijk, G. \& Prakken, H. Credulous and Sceptical Argument Games for Preferred Semantics. {\em JELIA}. \textbf{1919} pp. 239-253 (2000)
\bibitem{DUNG2007642}Dung, P., Mancarella, P. \& Toni, F. Computing ideal sceptical argumentation. {\em Artificial Intelligence}. \textbf{171}, 642-674 (2007)
\bibitem{ZhangL13}Zhang, X. \& Lin, Z. An argumentation framework for description logic ontology reasoning and management. {\em J. Intell. Inf. Syst.}. \textbf{40}, 375-403 (2013)
\bibitem{lacave2004}Lacave, C. \& Diez, F. A review of explanation methods for heuristic expert systems. {\em The Knowledge Engineering Review}. (2024)
\bibitem{LUKASIEWICZ2022103685}Lukasiewicz, T., Malizia, E., Martinez, M., Molinaro, C., Pieris, A. \& Simari, G. Inconsistency-tolerant query answering for existential rules. {\em Artificial Intelligence}. (2022)
\bibitem{Thomas2022Neg}Lukasiewicz, T., Malizia, E. \& Molinaro, C. Explanations for Negative Query Answers under Inconsistency-Tolerant Semantics. {\em Proceedings Of IJCAI}. (2022)
\bibitem{Lukasiewicz2020}Lukasiewicz, T., Malizia, E. \& Molinaro, C. Explanations for Inconsistency-Tolerant Query Answering under Existential Rules. {\em The Thirty-Fourth AAAI Conference On Artificial Intelligence}. pp. 2909-2916 (2020)
\bibitem{ARIOUA201776}Arioua, A., Croitoru, M. \& Vesic, S. Logic-based argumentation with existential rules. {\em Int. J. Approx. Reason.}. \textbf{90} pp. 76-106 (2017)
\bibitem{Arioua2016}Arioua, A. \& Croitoru, M. Dialectical Characterization of Consistent Query Explanation with Existential Rules. {\em Proceedings Of The Twenty-Ninth International Florida Artificial Intelligence Research Society Conference, FLAIRS}. (2016)
\bibitem{Arioua2015}Arioua, A., Tamani, N. \& Croitoru, M. Query Answering Explanation in Inconsistent Datalog\(\pm\) Knowledge Bases. {\em In DEXA}. \textbf{9261} pp. 203-219 (2015)
\bibitem{Meghyn2019}Bienvenu, M., Bourgaux, C. \& Goasdoué, F. Computing and Explaining Query Answers over Inconsistent DL-Lite Knowledge Bases. {\em J. Artif. Intell. Res.}. \textbf{64} pp. 563-644 (2019)
\bibitem{prakken_2006}Prakken, H. Formal systems for persuasion dialogue. {\em Knowl. Eng. Rev.}. \textbf{21}, 163-188 (2006)
\bibitem{Alrabbaa2022}Alrabbaa, C., Borgwardt, S., Koopmann, P. \& Kovtunova, A. Explaining Ontology-Mediated Query Answers Using Proofs over Universal Models. {\em RuleML+RR}. \textbf{13752} pp. 167-182 (2022)
\bibitem{Nielsen2007}Nielsen, S. \& Parsons, S. A Generalization of Dung's Abstract Framework for Argumentation: Arguing with Sets of Attacking Arguments. {\em Argumentation In Multi-Agent Systems}. pp. 54-73 (2007)
\bibitem{CALI201257}Calì, A., Gottlob, G. \& Lukasiewicz, T. A general Datalog-based framework for tractable query answering over ontologies. {\em Jour. Of Web Semantics}. \textbf{14} pp. 57-83 (2012)
\bibitem{Halpern1996}Halpern, J. Defining Relative Likelihood in Partially-Ordered Preferential Structures. {\em Procceeding Of UAI}. (1996)
\bibitem{Cayrol2014}Cayrol, C., Dubois, D. \& Touazi, F. On the Semantics of Partially Ordered Bases. 
\bibitem{Modgil2009}Modgil, S. \& Caminada, M. Proof Theories and Algorithms for Abstract Argumentation Frameworks.  (2009)
\bibitem{Cristhian15}Deagustini, C., Martinez, M., Falappa, M. \& Simari, G. On the Influence of Incoherence in Inconsistency-tolerant Semantics for Datalog\(\pm\). {\em IJCAI}. (2015)
\bibitem{AMGOUD2014}Amgoud, L. \& Vesic, S. Rich preference-based argumentation frameworks. {\em International Journal Of Approximate Reasoning}. \textbf{55}, 585-606 (2014)
\bibitem{kaci2021}Kaci, S., Der Torre, L., Vesic, S. \& Villata, S. Preference in Abstract Argumentation. {\em Handbook Of Formal Argumentation, Volume 2}. (2021)
\bibitem{ARIOUA2017244}Arioua, A., Buche, P. \& Croitoru, M. Explanatory dialogues with argumentative faculties over inconsistent knowledge bases. {\em Expert Systems With Applications}. \textbf{80} pp. 244-262 (2017)
\bibitem{Arioua2014FE}Arioua, A., Tamani, N., Croitoru, M. \& Buche, P. Query Failure Explanation in Inconsistent Knowledge Bases Using Argumentation. {\em Comma}. (2014)
\bibitem{Yun2020SetsOA}Yun, B., Vesic, S. \& Croitoru, M. Sets of Attacking Arguments for Inconsistent Datalog Knowledge Bases. {\em Comma}. (2020)
\bibitem{DUNNE2003221}Dunne, P. \& Bench-Capon, T. Two party immediate response disputes: Properties and efficiency. {\em Artificial Intelligence}. \textbf{149}, 221-250 (2003)
\bibitem{Cayrol2001}Cayrol, C., Doutre, S. \& Mengin, J. Dialectical Proof Theories for the Credulous Preferred Semantics of Argumentation Frameworks. {\em ECSQARU, Proceedings}. pp. 668-679 (2001)
\bibitem{Arieli2015}Arieli, O. \& Straßer, C. Sequent-based logical argumentation. {\em Argument Comput.}. \textbf{6}, 73-99 (2015)
\bibitem{AgostinoM18}D'Agostino, M. \& Modgil, S. Classical logic, argument and dialectic. {\em Artif. Intell.}. \textbf{262} pp. 15-51 (2018)
\bibitem{AmgoudB13}Amgoud, L. \& Besnard, P. Logical limits of abstract argumentation frameworks. {\em J. Appl. Non Class. Logics}. \textbf{23}, 229-267 (2013)
\bibitem{SCHULZ_TONI_2016}Schulz, C. \& Toni, F. Justifying answer sets using argumentation. {\em Theory And Practice Of Logic Programming}. \textbf{16}, 59-110 (2016)
\bibitem{Prakken05}Prakken, H. Coherence and Flexibility in Dialogue Games for Argumentation. {\em J. Log. Comput.}. \textbf{15}, 1009-1040 (2005)
\bibitem{DUNG2006114}Dung, P., Kowalski, R. \& Toni, F. Dialectic proof procedures for assumption-based, admissible argumentation. {\em Artif. Intell.}. \textbf{170}, 114-159 (2006)
\bibitem{Dung2009}Dung, P., Kowalski, R. \& Toni, F. Assumption-Based Argumentation. {\em Argumentation In Artificial Intelligence}. pp. 199-218 (2009)
\bibitem{Alejandro2014}Garcia, A. \& Simari, G. Defeasible logic programming: DeLP-servers, contextual queries, and explanations for answers. {\em Argument Comput.}. \textbf{5}, 63-88 (2014)
\bibitem{Prakken2002}Prakken, H. \& Vreeswijk, G. Logics for Defeasible Argumentation. {\em Handbook Of Philosophical Logic}. (2002)
\bibitem{Xiuyi14}Fan, X. \& Toni, F. A general framework for sound assumption-based argumentation dialogues. {\em Artif. Intell.}. \textbf{216} pp. 20-54 (2014)
\bibitem{ThangDH12}Thang, P., Dung, P. \& Hung, N. Towards Argument-based Foundation for Sceptical and Credulous Dialogue Games. {\em Proceedings Of COMMA}. \textbf{245} pp. 398-409 (2012)
\bibitem{ThangDH09}Thang, P., Dung, P. \& Hung, N. Towards a Common Framework for Dialectical Proof Procedures in Abstract Argumentation. {\em J. Log. Comput.}. \textbf{19} (2009)
\bibitem{Castagna21}Castagna, F. A Dialectical Characterisation of Argument Game Proof Theories for Classical Logic Argumentation. {\em Proceedings Of AIxIA}. \textbf{3086} (2021)
\bibitem{DAgostinoM18}D'Agostino, M. \& Modgil, S. Classical logic, argument and dialectic. {\em Artif. Intell.}. \textbf{262} pp. 15-51 (2018)
\bibitem{loanho_2024}Ho, L. \& Schlobach, S. A General Dialogue Framework for Logic-based Argumentation. {\em Proceedings Of The 2nd International Workshop On Argumentation For EXplainable AI}. \textbf{3768} pp. 41-55 (2024)
\bibitem{DimopoulosD0R0W24}Dimopoulos, Y., Dvorák, W., König, M., Rapberger, A., Ulbricht, M. \& Woltran, S. Redefining ABA+ Semantics via Abstract Set-to-Set Attacks. {\em AAAI}. pp. 10493-10500 (2024)
\bibitem{Meghyn2020}Bienvenu, M. \& Bourgaux, C. Querying and Repairing Inconsistent Prioritized Knowledge Bases: Complexity Analysis and Links with Abstract Argumentation. {\em Proceedings Of KR}. pp. 141-151 (2020)
\bibitem{BorgAS17}Borg, A., Arieli, O. \& Straßer, C. Hypersequent-Based Argumentation: An Instantiation in the Relevance Logic RM. {\em Proceeding Of TAFA}. (2017)
\bibitem{Hunter2010}Hunter, A. Base Logics in Argumentation. {\em Proceedings Of COMMA}. \textbf{216} pp. 275-286 (2010)
\bibitem{Priest89}Priest, G. Reasoning About Truth. {\em Artif. Intell.}. \textbf{39}, 231-244 (1989)
\bibitem{Belnap1977}Belnap, N. A Useful Four-Valued Logic. {\em Modern Uses Of Multiple-Valued Logic}. pp. 5-37 (1977)
\bibitem{BesnardH01}Besnard, P. \& Hunter, A. A logic-based theory of deductive arguments. {\em Artif. Intell.}. \textbf{128}, 203-235 (2001)
\bibitem{HeyninckA20}Heyninck, J. \& Arieli, O. Simple contrapositive assumption-based argumentation frameworks. {\em Int. J. Approx. Reason.}. \textbf{121} pp. 103-124 (2020)
\bibitem{Amgoud12}Amgoud, L. Five Weaknesses of ASPIC +. {\em  IPMU 2012 Proceedings}. \textbf{299} pp. 122-131 (2012)
\bibitem{ModgilP14}Modgil, S. \& Prakken, H. The ASPIC+ framework for structured argumentation: a tutorial. {\em Argument Comput.}. \textbf{5}, 31-62 (2014)
\bibitem{ArieliH24}Arieli, O. \& Heyninck, J. Collective Attacks in Assumption-Based Argumentation. {\em Proceedings Of The 39th ACM/SIGAPP Symposium On Applied Computing,SAC}. pp. 746-753 (2024)
\bibitem{CaminadaA07}Caminada, M. \& Amgoud, L. On the evaluation of argumentation formalisms. {\em Artif. Intell.}. \textbf{171}, 286-310 (2007)
\bibitem{HEYNINCK2020103}Jesse Heyninck, O. Simple contrapositive assumption-based argumentation frameworks. {\em International Journal Of Approximate Reasoning}. \textbf{121} pp. 103-124 (2020)
\bibitem{KrotzschRS15}Krötzsch, M., Rudolph, S. \& Schmitt, P. A closer look at the semantic relationship between Datalog and description logics. {\em Semantic Web}. \textbf{6}, 63-79 (2015)
\bibitem{Rapberger2024}Rapberger, A., Ulbricht, M. \& Toni, F. On the Correspondence of Non-flat Assumption-based Argumentation and Logic Programming with Negation as Failure in the Head. {\em CoRR}. \textbf{abs/2405.09415} (2024)
\bibitem{Lehtonen2024}Lehtonen, T., Rapberger, A., Toni, F., Ulbricht, M. \& Wallner, J. Instantiations and Computational Aspects of Non-Flat Assumption-based Argumentation. {\em CoRR}. \textbf{abs/2404.11431} (2024)
\bibitem{ArieliS19}Arieli, O. \& Straßer, C. Logical argumentation by dynamic proof systems. {\em Theor. Comput. Sci.}. \textbf{781} pp. 63-91 (2019)
\bibitem{AlsinetBG10}Alsinet, T., Béjar, R. \& Godo, L. A characterization of collective conflict for defeasible argumentation. {\em Computational Models Of Argument: Proceedings Of COMMA 2010}. \textbf{216} pp. 27-38 (2010)
\bibitem{Amgoud2009}Amgoud, L. \& Besnard, P. Bridging the Gap between Abstract Argumentation Systems and Logic. {\em Scalable Uncertainty Management}. pp. 12-27 (2009)
\bibitem{Stephen1975}Bloom, S. Some Theorems on Structural Consequence Operations. {\em Studia Logica: An International Journal For Symbolic Logic}. \textbf{34}, 1-9 (1975)
\bibitem{HechamBC17}Hecham, A., Bisquert, P. \& Croitoru, M. On the Chase for All Provenance Paths with Existential Rules. {\em Rules And Reasoning - International Joint Conference, RuleML+RR}. \textbf{10364} pp. 135-150 (2017)

\end{thebibliography}

\end{document}